\documentclass{article} % For LaTeX2e
\usepackage{iclr2021_conference,times}
\usepackage{url}

\usepackage{xr-hyper}
\usepackage{hyperref}
\usepackage{nicefrac}
\usepackage{graphicx}
\usepackage{wrapfig}
\usepackage{tabularx}
\usepackage{amsmath}
\usepackage{mathtools}
\usepackage{amsthm}
\usepackage{microtype}
\usepackage{subfigure}
\usepackage{xcolor}
\usepackage{amssymb}
\usepackage{bbm}
\usepackage{algorithm} %algorithm environments
\usepackage[noend]{algorithmic}
\usepackage{rotating} % pour la matrice geante
\usepackage{kbordermatrix} %border matrix
\usepackage[inline]{enumitem}
\usepackage{tablefootnote}
\usepackage{mathabx}
\usepackage{footnote}
\usepackage{adjustbox}
\usepackage{caption}
\usepackage{tikz}
\usetikzlibrary{automata,arrows,positioning,calc}
\usepackage{dblfloatfix} % fix for bottom-placement of figure
\usepackage{xparse}
\usepackage{multirow}
\usepackage{svg}

\newtheorem{corollary}{Corollary}[section] 
\newtheorem*{corollary*}{Corollary}
\newtheorem{definition}{Definition}[section]
\newtheorem{theorem}{Theorem}[section] 
\newtheorem*{theorem*}{Theorem}
\newtheorem{lemma}{Lemma}[section] 
\newtheorem*{lemma*}{Lemma}
\newtheorem{proposition}{Proposition}[section] 
\newtheorem*{proposition*}{Proposition}

\newtheorem{remark}{Remark}[section]
\newtheorem{example}{Example}[section]

\makeatletter
\newcommand*{\addFileDependency}[1]{% argument=file name and extension
  \typeout{(#1)}
  \@addtofilelist{#1}
  \IfFileExists{#1}{}{\typeout{No file #1.}}
}
\makeatother

\DeclareMathOperator{\St}{\mathcal{S}}
\DeclareMathOperator{\A}{\mathcal{A}}
\DeclareMathOperator{\X}{\mathcal{X}}

\DeclarePairedDelimiter\abs{\lvert}{\rvert}% | |
\DeclarePairedDelimiter\norm{\lVert}{\rVert}% || ||
\DeclarePairedDelimiter\ceil{\lceil}{\rceil}% || ||
\renewcommand{\kbldelim}{(}% Left delimiter
\renewcommand{\kbrdelim}{)}% Right delimiter

\newcommand \PiD{\bar{\Pi}_m}
\newcommand \piD{\bar{\pi}}

 \newcommand{\ie}{i.\,e., }
 \newcommand{\eg}{e.\,g., }
 
\renewcommand{\arraystretch}{1.5}

\newcommand{\algorithmfootnote}[2][\footnotesize]{%
  \let\old@algocf@finish\@algocf@finish% Store algorithm finish macro
  \def\@algocf@finish{\old@algocf@finish% Update finish macro to insert "footnote"
    \leavevmode\rlap{\begin{minipage}{\linewidth}
    #1#2
    \end{minipage}}%
  }%
}

\title{Acting in Delayed Environments with \\ Non-Stationary Markov Policies
}

\author{Esther Derman\thanks{Equal contribution} \\
Technion\\
\texttt{estherderman@campus.technion.ac.il} \\
\And
Gal Dalal\footnotemark[1] \\
Nvidia Research \\
\texttt{gdalal@nvidia.com} \\
\And
Shie Mannor \\
Nvidia Research \& Technion \\
\texttt{shie@ee.technion.ac.il} \\
}

\iclrfinalcopy  
\begin{document}

\maketitle
\vspace{-0.7cm}
\begin{abstract}
The standard Markov Decision Process (MDP) formulation hinges on the assumption that an action is executed immediately after it was chosen. However, assuming it is often unrealistic and can lead to catastrophic failures in applications such as robotic manipulation, cloud computing, and finance. We introduce a framework for learning and planning in MDPs where the decision-maker commits actions that are executed with a delay of $m$ steps. The brute-force state augmentation baseline where the state is concatenated to the last $m$ committed actions suffers from an exponential complexity in $m$, as we show for policy iteration. We then prove that with execution delay, deterministic Markov policies in the original state-space are sufficient for attaining maximal reward, but need to be non-stationary. As for stationary Markov policies, we show they are sub-optimal in general. Consequently, we devise a non-stationary Q-learning style model-based algorithm that solves delayed execution tasks without resorting to state-augmentation. Experiments on tabular, physical, and Atari domains reveal that it converges quickly to high performance even for substantial delays, while standard approaches that either ignore the delay or rely on state-augmentation struggle or fail due to divergence. The code is available at \small{\url{github.com/galdl/rl_delay_basic}} and \small{\url{github.com/galdl/rl_delay_atari}}.
\end{abstract}

\vspace{-0.5cm}
\section{Introduction}
\label{section:intro}
\vspace{-0.25cm}
The body of work on reinforcement learning (RL) and planning problem setups has grown vast in recent decades. Examples for such distinctions are different objectives and constraints, assumptions on access to the model or logged trajectories, on-policy or off-policy paradigms, etc. \citep{puterman2014markov}. However, the study of delay in RL remains scarce. It is almost always assumed the action is executed as soon as the agent chooses it. This assumption seldom holds in real-world applications \citep{dulac2019challenges}. Latency in action execution can either stem from the increasing computational complexity of modern systems and related tasks, or the infrastructure itself. The wide range of such applications includes robotic manipulation, cloud computing, financial trading, sensor feedback in autonomous systems, and more. 
To elaborate, consider an autonomous vehicle required for immediate response to a sudden hazard on the highway. Driving at high speed, it suffers from perception module latency when inferring the surrounding scene, as well as delay in actuation once a decision has been made. While the latter phenomenon is an instance of \emph{execution delay}, the former corresponds to \emph{observation delay}. These two types of delay are in fact equivalent and can thus be treated with the same tools \citep{katsikopoulos2003markov}.  

{ \bf Related works. }
The notion of delay is prominent in control theory with linear time-invariant systems \citep{bar1995explicit, dugard1998stability, richard2003time, fridman2014introduction, bruder2009impulse}. 
 While the delayed control literature is vast, our work intersects with it mostly in motivation. 
In the above control theory formulations, the system evolves according to some known diffusion or stochastic differential equation. Differently, the discrete-time MDP framework does not require any structural assumption on the transition function or reward. 

A few works consider a delay in the reward signal rather than in observation or execution. Delayed reward has been studied on multi-armed bandits for deterministic and stochastic latencies \citep{joulani2013online} and for the resulting arm credit assignment problem \citep{pike2017bandits}. In the MDP setting, \citet{campbell2016multiple} proposed a Q-learning variant for reward-delay that follows a Poisson distribution. 
\cite{katsikopoulos2003markov} considered three types of delay: observation, execution, and reward. \cite{chen2020delaymultiagent} studied execution delay on multi-agent systems. The above works on MDPs employed state-augmentation with a primary focus on empirical evaluation of the degradation introduced by the delay. In this augmentation method, all missing information is concatenated with the original state to overcome the partial observability induced by the delay. 
 The main drawback of this embedding method is the exponential growth of the state-space with the delay value \citep{walsh2009learning, chen2020delay} and, in the case of \citep{chen2020delaymultiagent}, an additional growth that is polynomial with the number of agents. 
 
\citet{walsh2009learning} avoided state-augmentation in MDPs with delayed feedback via a planning approach. By assuming the transition kernel to be close to deterministic, their \emph{model-based simulation} (MBS) algorithm relies on a most-likely present state estimate. Since the Delayed-Q algorithm we devise here resembles to MBS in spirit, we highlight crucial differences between them:
First, MBS is a conceptual algorithm that requires the state-space to be finite or discretized. This makes it highly sensitive to the state-space size, as we shall demonstrate in Sec.~\ref{section: experiments}[Fig.~\ref{fig:varying_maze_size}], prohibiting it from running on domains like Atari. Differently, Delayed-Q works with the original, possibly continuous state-space. Second, MBS is an offline algorithm: it estimates a surrogate, non-delayed MDP from samples, and only then does it solve that MDP to obtain the optimal policy \citep{walsh2009learning}[Alg.~2, l.~16]. This is inapplicable to large continuous domains and is again in contrast to Delayed-Q.

Recent studies considered a concurrent control setting where action sampling occurs simultaneously with state transition \citep{ramstedt2019real, xiao2020thinking}. Both assumed a single action selection between two consecutive observations, thus reducing the problem to an MDP with execution delay of $m=1$. \cite{chen2020delay} have generalized it to an arbitrary number of actions between two observations. \cite{hester2013texplore} addressed execution delay in the braking control of autonomous vehicles with a relatively low delay of $m\leq 3.$ All these works employ state-augmentation to preserve the Markov property of the process, whereas we are interested whether this restriction can be lifted. Additionally, they studied policy-gradient (policy-based) methods, while we introduce a Q-learning style (value-based) algorithm. Likewise, \cite{firoiu2018human} proposed a modified version of the policy-based IMPALA \citep{espeholt2018impala} which is evaluated on a single video game with delay values of $m\leq 7$.
To the best of our knowledge, our work is the first to tackle a delayed variant of the popular Atari suite \citep{bellemare2013arcade}.

{\bf Contributions.} Revisiting RL with execution delay both in theory and practice, we introduce:
    \begin{enumerate}
        \item Analysis of a delayed MDP quantifying the trade-off between stochasticity and delay.
        \item The first tight upper and lower complexity bounds on policy iteration for action-augmented MDPs. We stress that this is also a contribution to general RL theory of non-delayed MDPs. 
        \item A new formalism of execution-delay MDPs that avoids action-embedding. Using it, we  prove
        that out of the larger set of history-dependent policies, restricting to non-stationary deterministic Markov policies  is sufficient for optimality in delayed MDPs. We also derive a Bellman-type recursion for a \emph{delayed value function}.
        \item A model-based DQN-style algorithm that yields non-stationary Markov policies. Our algorithm outperforms the alternative standard and state-augmented DDQN in 39 of 42 experiments spanning over 3 environment categories and delay of up to $m=25$.
    \end{enumerate}
 
 \vspace{-0.4cm}
\section{Preliminaries: Non-Delayed Standard MDP}
\label{sec:prelim}
\vspace{-0.2cm}
Here, we describe the standard non-delayed MDP setup. Later, in Sec.~\ref{section: ed-mdp new formulation}, we  introduce its generalization to the delayed case. We follow and extend  notations from \citep{puterman2014markov}[Sec. 2.1.]. 
	An infinite horizon discounted MDP is a tuple $(\St, \A, P, r, \gamma)$ where $\St$ and  $\A$ are finite state and action spaces, 
	$P : \St\times \A \rightarrow \Delta_{\St}$ is a transition kernel,
	the reward $r:\St\times\A\rightarrow\mathbb{R}$ is a bounded function, and $\gamma\in [0,1)$ is a discount factor. 
	At time $t$, the agent is in $s_t$ and draws an action $a_t$ according to a decision rule $d_t$ that maps past information to a probability distribution $q_{d_t}$ over the action set. Once $a_t$ is taken,
	the agent receives a reward $r(s_t,a_t).$ 
	
	A decision rule can be history-dependent (H) or Markovian (M) , and randomized (R) or deterministic (D). Denote by $\mathcal{H}_t$ the set of possible histories up to time $t.$ Then, a history-dependent decision-rule is given by $d_t: \mathcal{H}_t \rightarrow \Delta_{\A}$ with $h_t \mapsto q_{d_t(h_t)}(\cdot).$ A Markovian decision-rule, on the other hand, maps states to actions, i.e., $d_t: \St \rightarrow \Delta_{\A}$ with $s \mapsto q_{d_t(s)}(\cdot)$.
    A \emph{policy} $\pi:= (d_t)_{t\geq 0}$ is a sequence of decision rules whose type dictates that of the policy. It can be either Markovian deterministic ($\Pi^\textsc{MD}$) or randomized ($\Pi^\textsc{MR}$), history-dependent deterministic ($\Pi^\textsc{HD}$) or randomized ($\Pi^\textsc{HR}$). It is stationary 
    if its decision rules do not depend on time, \ie $d_t = d$ for all $t\geq 0$. This defines the smaller class of stationary policies: deterministic ($\Pi^\textsc{SD}$) and randomized ($\Pi^\textsc{SR}$). Note that stationary policies are inherently Markovian. Indeed, at time $t=0$, $d: \mathcal{H}_0\rightarrow \Delta_{\A}$ is state-dependent because $\mathcal{H}_0 = \St$. Since the policy is stationary, \ie $d_t =d ~\forall t$, subsequent decision rules are also state-dependent, thus Markovian. This makes $\Pi^\textsc{HR}$ the most general set and $\Pi^\textsc{SD}$ the most specific.

    We denote probability model by $\mathbb{P}_0^{\pi}$, where the subscript $0$ stands for the delay value $m=0$.  The related random variables are denoted by $\tilde{s}_t \in \St, \tilde{a}_t \in \A$ and $\tilde{h}_t \in (\St\times\A)^t\times \St.$ 
	The value function given policy $\pi\in\Pi^{\textsc{HR}}$ is defined as
	$v^{\pi}(s) = \mathbb{E}_0^{\pi}\left[ \sum_{t = 0}^{\infty}\gamma^{t}r(\tilde{s}_t, \tilde{a}_t) \biggm | \tilde{s}_0 = s\right]$, where the expectation is taken with respect to (w.r.t.) $\mathbb{P}_0^{\pi}(\cdot|\tilde{s}_0=s)$. 
	Let the optimal value function 
\begin{equation}
    \label{def: optimal value}
    v^*(s) := \max_{\pi\in\Pi^{\textsc{HR}}} v^\pi(s), \quad \forall s \in \St.
\end{equation}
Our goal is to find a policy $\pi^*$ that yields $v^*$,
and it is known that focusing on stationary deterministic policies $\pi\in\Pi^{\textsc{SD}}$ is sufficient for reaching the optimum in \eqref{def: optimal value} \citep{puterman2014markov}[Thm. 6.2.10.].

\section{MDPs with Delay: A Degradation Example}
\label{sec: ED MDP}
\begin{wrapfigure}{R}{.36\textwidth}
\vspace{-0.55cm}
  \centering
\begin{tikzpicture}[->, >=stealth', auto, semithick,on grid, node distance=3cm]
\tikzstyle{every state}=[fill=white,draw=black,thick,text=black,scale=1]
\node[state, text width=1.5cm]    (A) {$\qquad s_0$ \footnotesize{$a_0: r=1$ $a_1: r=0$}};
\node[state, text width=1.5cm]    (B)[right of=A]   {$\qquad s_1$ \footnotesize{$a_0: r=0$ $a_1: r=1$}};
\path
(A) edge[bend left, above]     node{$p$}     (B)
    edge[loop above, right]    node[yshift = -.1cm, xshift = .2cm]{$1-p$}      (A)
(B) edge[bend left]               node{$p$}           (A)
    edge[loop above, left]       node[yshift = -.1cm, xshift = -.2cm]{$1-p$}         (B);
\end{tikzpicture}
 \captionof{figure}{Degradation due to delay in a two-state MDP.} 
 \label{fig: degradation example}
\end{wrapfigure}
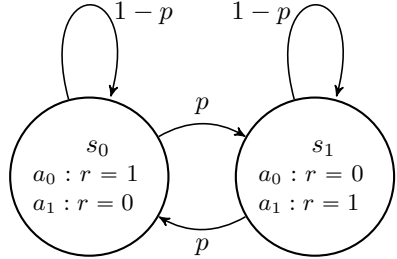
\vspace{-2mm}
In an MDP with execution delay\footnote{The exact terminology used by \citet{katsikopoulos2003markov} is \emph{action delay}, while in \citep{bertsekas1995dynamic}[Section~1.4] it is \emph{time lag}. We prefer the term \emph{execution delay} since the action is itself decided instantaneously.} $m$, any action chosen at time $t$ is executed at $t+m$. Therefore, at each step, the agent witnesses the current state and action being executed, but selects a new action that will be applied in a future state. We assume that $m$ decided actions are already awaiting execution at $t=0$, so at any given time, the queue of pending actions is of constant length $m$.
As we illustrate in the next example, having a  delay generally comes at a price. 

\begin{example}[Two-state MDP]
\label{exmp: 2-state MDP}
	Consider the MDP in Fig.~\ref{fig: degradation example}. It has two states and two actions: $\St = \{s_0, s_1\},$  $\A = \{a_0,a_1\}.$ The transition kernel is independent of the action: for all $s,s' \in \St \mbox{ s.t. } s\neq s', P(s'|s,a)=P(s'|s) = p$ where $p\in[0.5, 1]$. The reward is positive for one of the two actions only: $r(s_0,a_0) =r(s_1,a_1)= 1,$ $r(s_0,a_1) =r(s_1,a_0)= 0.$ 
\end{example}

We inspect the return obtained from the commonly used set of stationary deterministic policies $\Pi^{\text{SD}}$. As expected, the highest possible return is attained when $m=0,$ but monotonically decreases with the delay, $m,$ and increases with the level of certainty, $p.$ We analytically quantify this effect in the following and give a proof in Appx.~\ref{sec: example proof}.

    \begin{proposition}
    \label{prop: degredation}
    For delay $m\in \mathbb{N}$ and $p\in[0.5, 1]$, the optimal return of $\pi^* \in \Pi^{\textsc{SD}}$ is
    $
    \frac{1 + (2p - 1)^m}{2(1-\gamma)}.
    $

    \end{proposition}
        \begin{remark}
        \label{remark:noise delay tradeoff}
    This result demonstrates a clear tradeoff between stochasticity and delay. For $p \rightarrow 0.5$ or $m \rightarrow \infty,$ the return goes to its minimal value of $0.5/(1-\gamma).$ Contrarily, for $p \rightarrow 1$ or $m \rightarrow 0,$ it goes to its maximal value of $1/(1-\gamma).$
    \end{remark}
    
\vspace{-0.4cm}
\section{The Augmentation Approach}
	\label{section:med}
	\vspace{-0.25cm}
	In this section, we consider state-augmentation for solving MDPs with execution delay. We begin with defining an equivalent MDP with a larger state space that memorizes all missing information for an informed decision. Due to the full observability, the resulting optimal augmented policy attains the optimal return in the original delayed MDP.
	\begin{definition}[$m$-AMDP]
\label{def: EDMDP}
Given MDP $(\St, \A, P, r, \gamma)$ and $m\in\mathbb{N}$, an $m$-Augmented MDP ($m$-AMDP) is a tuple $(\X_m, \A, F,g, \gamma)$ such that $\X_m := \St\times\A^m$ is the augmented state-space, $\A$  the original action-space, $F$ is the transition matrix given in Appx.~\ref{appendix: augmented mdp construction},\eqref{eq:transition_F}, and $g$ is the reward function given in Appx.~\ref{appendix: augmented mdp construction}, \eqref{eq:reward_function}.
\end{definition}

The pending action queue is concatenated to the original state to form an augmented state $x_t := (s_t, a_{t}^{-1}, \cdots, a_{t}^{-m})\in\X_m,$
where $a_{t}^{-i}$ is the $i$-th pending action at time $t$. It means that in the following step, $t+1$, action $a_t^{-m}$ will be executed independently of the present action selection, the queue will shift to the right, and the newly selected action will be at the second coordinate.  By construction, the $m$-AMDP is non-delayed; it directly accounts for execution delay through its state-representation, as opposed to our coming formulation in Sec.~\ref{section: ed-mdp new formulation}.
We further define a stationary deterministic policy $\piD \in \PiD^{\textsc{SD}}$ with corresponding decision rule $\bar{d}: \X_m \rightarrow \Delta_{\A}$ and \textit{augmented value function} 
$
    v^{\piD}(x) := \mathbb{E}^{\piD}\left[ \sum_{t = 0}^{\infty}\gamma^{t}g(\tilde{x}_t, \tilde{a}_t) | \tilde{x}_0 = x\right].
$ As in \eqref{def: optimal value}, our goal is to solve
  $\bar{v}^*(x) = \max_{\piD\in\PiD^{\textsc{SD}}} v^{\piD}(x), \quad\forall x\in\X_m$.

We now analyze the classical Policy Iteration (PI) algorithm \citep{howard1960dynamic}  for $m$-augmented MDPs and provide a finite-time analysis of its convergence. We refer to it as $m$A-PI and provide its pseudo-code in Appx.~\ref{sec: med_pi_algo}. 
We consider PI since it is a canonical representative upon which many other algorithms are built. Admittedly, we did not find any other formal result quantifying the effect of augmentation on a planning or learning algorithm, other than a PAC upper bound for R-max with $\epsilon$-optimal policies  \citep{walsh2009learning}. 
A proof for the next result is given in Appx.~\ref{sec: general lower bound}.   

\begin{theorem}[Lower Bound for $m$A-PI]
\label{theorem: maPI lower bound}
The number of iterations required for $m$A-PI 
to converge in  $m$-AMDP $\mathcal{M}_m$ is $\Omega(|\X_m|) = \Omega(|\St||\A|^m).$
\end{theorem}

 Thm.~\ref{theorem: maPI lower bound} does not take advantage of the special delay problem structure but rather is an application of our more general result to augmented MDPs (Appx.\ref{sec: general lower bound}). As pointed out in \cite{scherrer2016improved}, the lower-bound complexity of PI is considered an open problem, at least in the most general MDP formulation. Lower-bounds have been derived in specific cases only, such as deterministic MDPs \citep{hansen2010lower}, total reward criterion \citep{fearnley2010exponential} or high discount factor \citep{hollanders2012complexity}. Even though we did not intend to directly address this open question, our lower bound result seems to be a contribution on its own to the general theory of non-delayed MDPs.  
        
Next, we show that the above lower bound is tight (up to a factor of $\abs{\A}$ and logarithmic terms) and $m$A-PI is guaranteed to converge after $\tilde{O}(\abs{\St}\abs{\A}^{m+1}).$ 
A proof is given in Appx.~\ref{sec: augmented upper bound}.
\begin{theorem}[$m$A-PI Convergence]
\label{thm:delay_PI_augmented}
The $m$A-PI algorithm 
converges to the optimal value-policy pair $(\bar{v}^*, \piD^*)$ in at most $\abs{\St}\abs{\A}^m(\abs{\A} - 1)\ceil*{\log\left(\nicefrac{1}{\gamma}\right)^{-1}\log\left(\nicefrac{1}{1-\gamma}\right)}$
iterations. 
\end{theorem}
\vspace{-0.45cm}
\section{Execution-Delay MDP: A New Formulation}
\label{section: ed-mdp new formulation}
\vspace{-0.25cm}
In this section, we introduce and study the stochastic process generated by an MDP with execution delay, without resorting to state-augmentation. 
In the ED-MDP we consider, the probability measure changes according to the delay value $m$. We assume that during the $m$ initial steps, actions are sequentially executed according to a \emph{fixed} queue $\bar{a}:= (\bar{a}_0,\cdots, \bar{a}_{m-1}) \in \A^m$. Unlike $m$-AMDPs, the initial queue of pending actions here plays the role of an exogenous variable that is not embedded into the state-space.
A policy $\pi\in\Pi^{\textsc{HR}}$ induces a probability measure $\mathbb{P}^{\pi}_m$ that is defined through a set of equations which, for brevity, we defer to Appx.~\ref{sec: EDMDP formulation}[\eqref{eq: init state distribution}-\eqref{eq: state transition function}].
We note that for $t<m$, decision rules do not depend on the history, while for $t\geq m$, they depend on the history up to $t-m$ only. Let $\mu$ be an initial state distribution and $\delta$ a Dirac distribution. Using this and the notations from Sec.~\ref{sec:prelim}, we can explicitly write the probability of a sample path. See proof in Appx.~\ref{appendix: delayed process distrib}. 

\begin{proposition}
\label{proposition: delayed process distribution}
For policy $\pi := (d_0,d_1,\cdots)\in \Pi^{\textsc{HR}}$, the probability of observing history $h_t :=(s_0, a_0,s_1,a_1 \cdots, a_{t-1}, s_t)$ is given by:
\begin{equation*}
\begin{split}
 &\mathbb{P}^{\pi}_m(\tilde{s}_0 = s_0, \tilde{a}_0 = a_0,\tilde{s}_1 = s_1, \tilde{a}_1 = a_1, \cdots, \tilde{a}_{t-1}  = a_{t-1}, \tilde{s}_t  = s_t) \\
 &=
\mu(s_0)\left(\prod_{k =0}^{m-1}\delta_{\bar{a}_k}(a_k) p(s_{k+1}|s_k, a_k)\right)\left(\prod_{k=m}^{t-1} q_{d_{k-m}(h_{k-m})}(a_k)p(s_{k+1}|s_k, a_k)
  \right). 
\end{split}
\end{equation*}
\end{proposition}

From Prop.~\ref{proposition: delayed process distribution} we deduce that, differently than the standard MDP setting where any Markov policy induces a Markov process, the delayed process is not Markovian even for stationary policies (see Appx.~\ref{appx: remark markov property} for a formal proof). 
Next, we show that for any history-dependent policy and starting state, there exists a Markov policy (not necessarily stationary) that generates the same process distribution. Consequently, despite execution delay, one can restrict attention to Markov policies without impairing performance. 

\begin{theorem}
\label{theorem: markov policy is sufficient}
Let $\pi \in \Pi^{\textsc{HR}}$ be a history dependent policy. For all $s_0\in\St$, there exists a Markov policy $\pi' \in \Pi^{\textsc{MR}}$  that yields the same process distribution as $\pi$, i.e., 
${
   \mathbb{P}^{\pi'}_m(\tilde{s}_{t-m}  = s', \tilde{a}_t  = a | \tilde{s}_0  = s_0) = 
\mathbb{P}^{\pi}_m(\tilde{s}_{t-m}  = s', \tilde{a}_t  = a | \tilde{s}_0  = s_0), \qquad \forall a\in\A, s'\in \St,t\geq m.
}$
\end{theorem}

The proof is given in Appx.~\ref{appendix: markov policy is sufficient}. It builds on the concept that for each history-dependent policy $\pi \in \Pi^{\textsc{HR}}$, one can choose a sequence of Markov decision rules that reconstruct the same time-dependent action distribution in the process induced by $\pi.$ 

This result proves attainability of the optimum over $\Pi^{\textsc{MR}}$, but not how one can efficiently find an optimal policy. In Appx.~\ref{appx: delayed value function}, \eqref{def: delayed value function}, we formally define the \emph{delayed value function} $v^{\mu_0:\mu_{m-1},\pi}_{m}$ for policy $\pi$ and initial action distribution queue $\mu_0:\mu_{m-1}:=(\mu_0,\dots,\mu_{m-1}).$ In Thm.~\ref{thm: delayed value function} there, we show that it satisfies a non-stationary Bellman-type recursion. Though the question of how to efficiently find an optimal non-stationary Markov policy remains generally open, we partially answer it by proving that a deterministic Markov policy is sufficient for the optimal delayed value function.
\begin{theorem}
\label{thm: optimal deterministic policy}
For any action distribution queue $\mu_0:\mu_{m-1}:=(\mu_0,\dots,\mu_{m-1})$ and $s_0\in \St,$
$$
\max_{\pi \in \Pi^{\textsc{MD}}} v^{\mu_0:\mu_{m-1},\pi}_{m}= \max_{\pi \in \Pi^{\textsc{MR}}}v^{\mu_0:\mu_{m-1},\pi}_{m}.
$$
\end{theorem}

\begin{wrapfigure}{R}{.45\textwidth}
    \centering
    \includegraphics[scale=0.26]{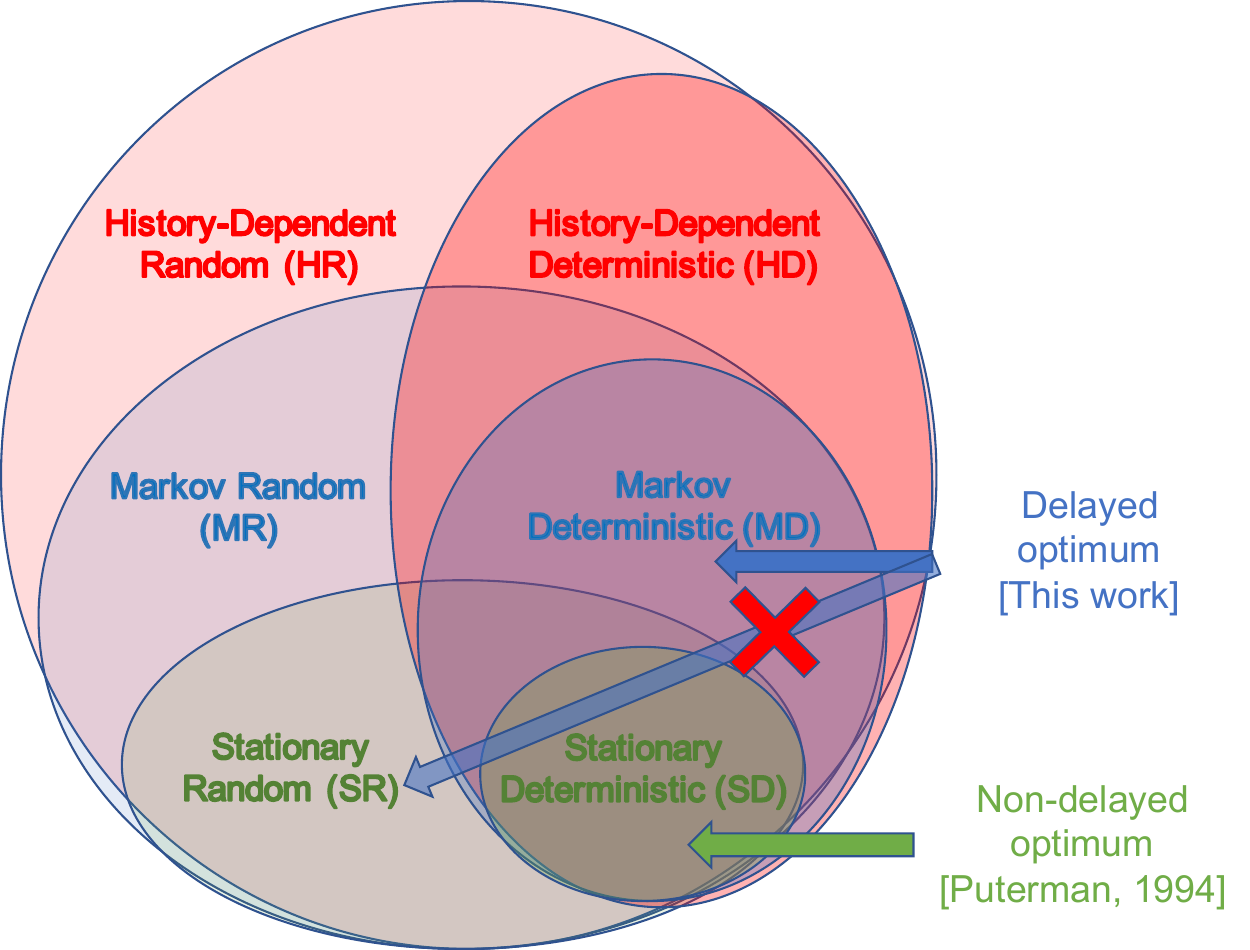}
    \caption{
    Optimality of policy types in ED-MDPs: Markovness is sufficient 
    but non-stationarity 
    is necessary.
    }
    \label{fig:diagram policies}
\end{wrapfigure}

\textbf{Degradation due to stationarity.} To complement the finding that a deterministic Markov policy can be optimal for any ED-MDP, we show that restricting to stationary policies impairs performance in general. Thus, while in non-delayed MDPs it is enough to focus on the latter, in ED-MDPs the restriction should be to the more general class of Markov policies.
\begin{proposition}
\label{prop: degradation due to stationarity}
There exists an $m$-ED-MDP for which all stationary policies are sub-optimal.
\end{proposition}
\vspace{-0.25cm}
This result follows from computing the optimal return for stationary and non-stationary policies in the ED-MDP from Example~\ref{exmp: 2-state MDP} using simulation. 
We elaborate on this further in Appx.~\ref{sec: degradation stationarity}. There, we also confirm that our theoretical return from Prop.~\ref{prop: degredation} matches closely with simulation. Lastly, a visualization of the results from this section is given in Fig.~\ref{fig:diagram policies}.

 \vspace{-0.25cm}
\section{A New Algorithm: Delayed-Q}
\label{sec: algorithm}

\begin{wrapfigure}{R}{.55\textwidth}
\vspace{-0.355cm}
 \centering
    \includegraphics[scale=0.17]{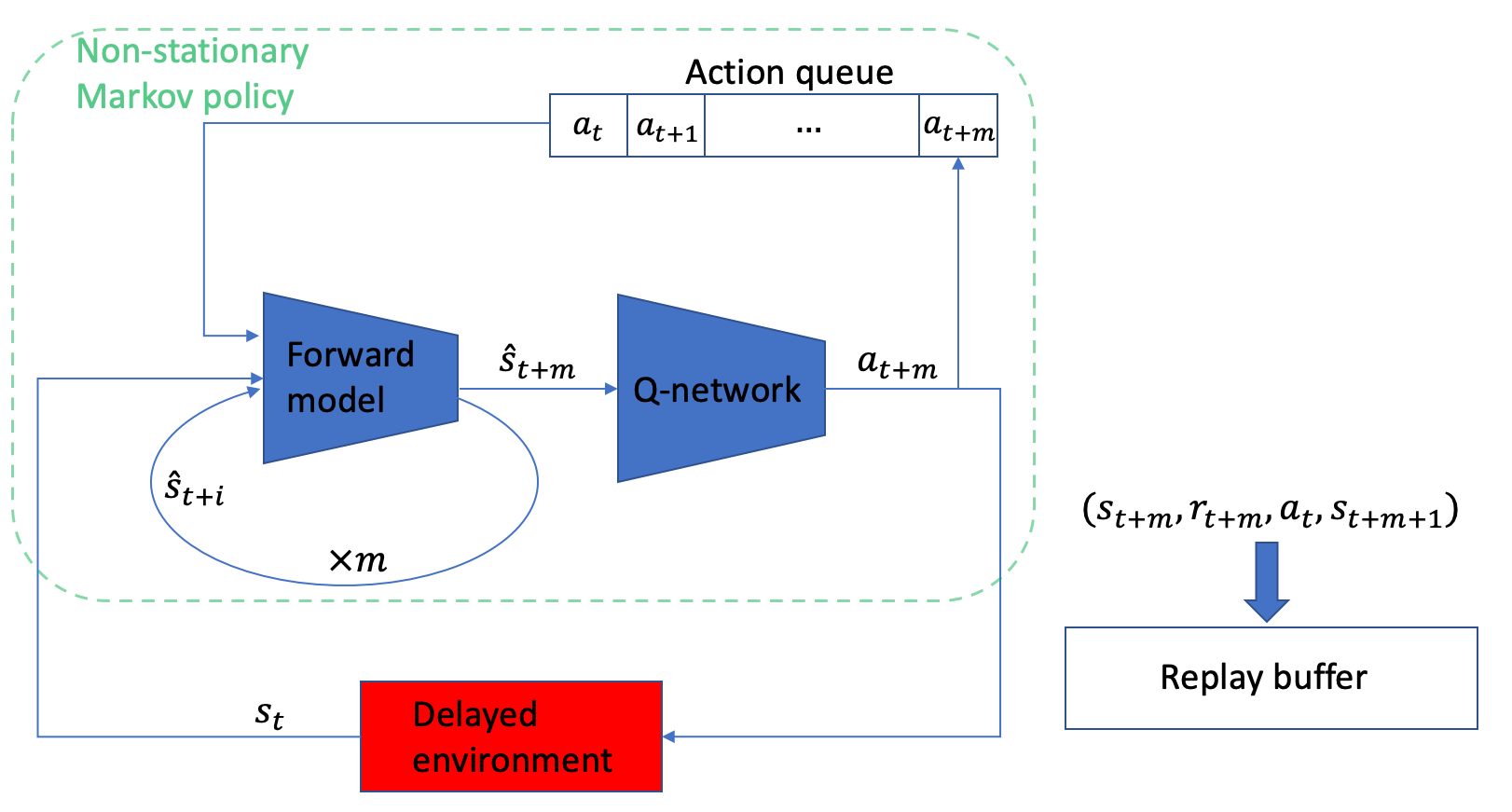}
    \captionof{figure}{Delayed-Q algorithm diagram.}
    \label{fig:delayed-q diagram}
\end{wrapfigure}

\vspace{-0.25cm}
We now introduce an algorithm capable of successfully handling tasks with execution delay by inferring the future $m$-step state before each decision.\newline
{\bf Algorithm Description.} Fig.~\ref{fig:delayed-q diagram} depicts the algorithm. As a first stage, to select an action $a_t$ to be executed in a future state $s_{t+m}$, we infer that future state $\hat{s}_{t+m}$ using the current state $s_t$ and the queue of pending actions $(a_{t-m},\dots,a_{t-1}).$ This is done by successively applying an approximate forward model $m$ times: $\hat{s}_{t+1}=f(s_t,a_{t-m}), \dots,\hat{s}_{t+m}=f(\hat{s}_{t+m-1},a_{t-1}).$ More details on the forward models are given in Sec.~\ref{section: experiments}. The approximate model here is simpler than other model-based algorithms such as tree-search methods, because it does not require access to the reward function. Also, only a single trajectory is sampled rather than exponentially many w.r.t. the horizon length. We do note this method benefits from the environment not being entirely stochastic \citep{walsh2009learning}. Still, as we show next, it performs well even on noisy environments. 
As a second stage, we select an action according to a policy $a_t=\pi(\hat{s}_{t+m}).$ The two stages of this procedure can be represented as a non-stationary Markov policy $\pi_t(s_t),$ where the non-stationarity stems from the time-dependency of the action queue, and the Markov property from the policy being applied on $s_t$ and no prior history. Notably, the Q-function here does not take past actions as input, contrarily to the augmentation approach in Sec.~\ref{section:med}. To better stress the non-stationarity, we note that applying the policy on the same state at different times can output different actions. Lastly, for training, we maintain a sample-buffer of length $m$ which we use to shift action $a_t$ into the tuple $(s_{t+m}, r_{t+m}, a_t, s_{t+m+1})$ prior to each insertion to the replay buffer.
During the course of this work, we also experimented with a model-free variant. Instead of `un-delaying' the Q-function with the forward-model, we defined a delayed Q-function trained on sequences whose actions were shifted $m$ steps forward. However, the obtained results were unsatisfactory, seemingly because the Q-function is unable to implicitly learn the $m$-step transition. 

{\bf Point-Estimate Approaches.} 
For completeness, we mention alternatives to using a `most-likely' state estimate, such as an expected future state. To demonstrate why point-estimate prediction can be devastating, consider an MDP where $s=(x,t)$: position and time, respectively. Starting from $s_0=(0,0)$, $t$ progresses deterministically, while $x$ behaves like a random walk with momentum; \ie if $x>0$, then $x+1$ is more likely than $x-1$, and vice versa. The process obviously diverges with time. 
Consider two actions: one is good when $|x|$ is big, and the other when $|x|$ is small. For a large delay $m,$ the PDF of the state is bi-modal and symmetric around $(Z,m)$ and $(-Z,m)$ for some finite $Z.$ But, a point estimate (\eg ML or MAP) would yield a value of $(0,m)$. In addition to this example, we observe that in our Ex.~\ref{exmp: 2-state MDP}, any alternative to a `most-likely' state estimate is worse: there, the optimal policy applies actions based on the most-likely state (see proof of Prop.~\ref{prop: degredation}), while it is easy to see that any other policy weighing future state probabilities leads to lower reward.

\vspace{-0.4cm}
\section{Experiments}
\label{section: experiments}
\begin{wrapfigure}{R}{.37\textwidth}
\vspace{-0.55cm}
 \centering
    \includegraphics[scale=0.2]{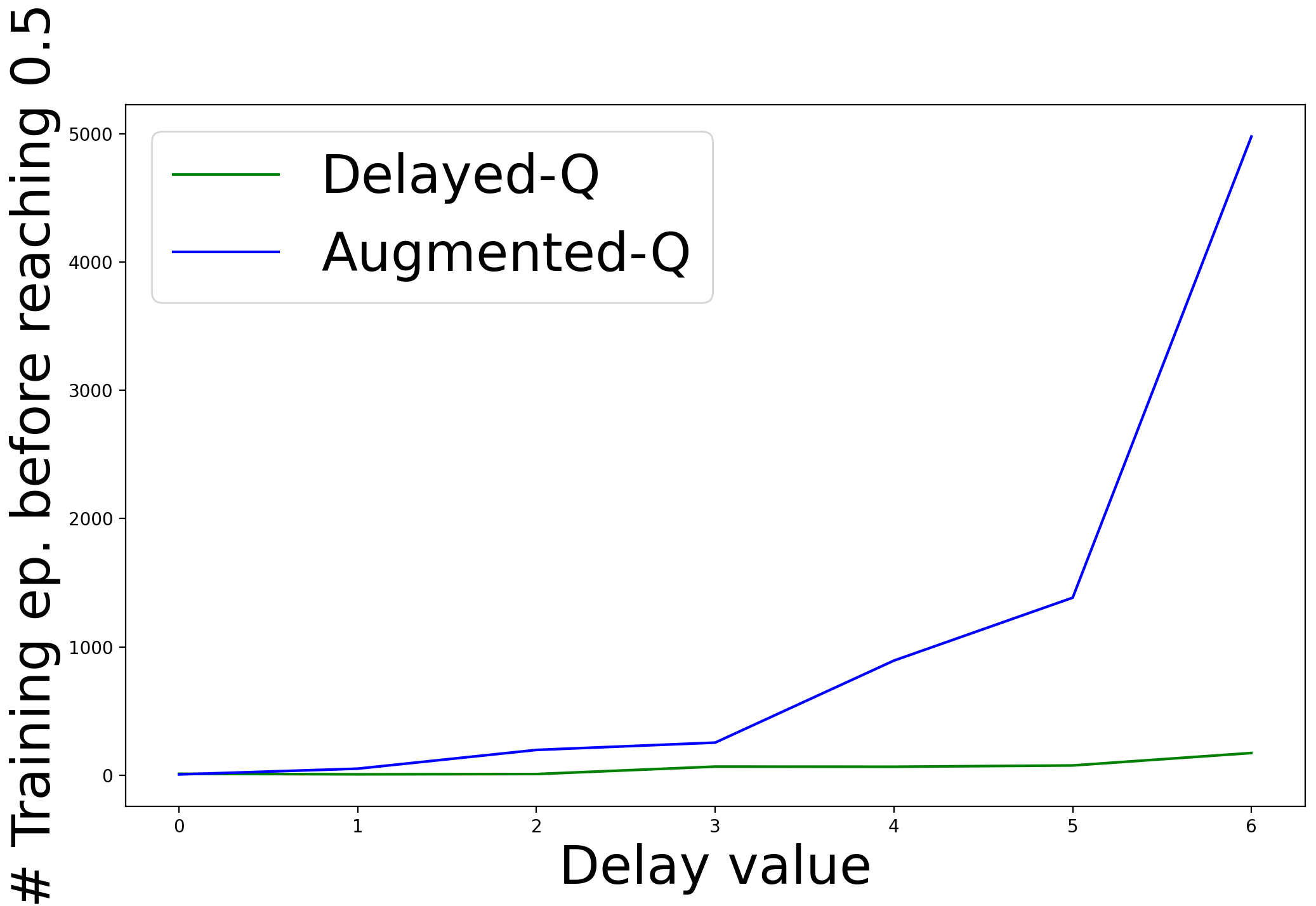}
    \captionof{figure}{Maze: Time complexity as a function of $m$}
    \label{fig:exp vs linear growth}
\end{wrapfigure}
\vspace{-0.25cm}
We perform experiments in a wide range of domains: tabular, physical, and image-based Atari. All of them include stochasticity: In the maze we inject noise to actions; in the physical domains we perturb the masses at each step; and Atari is stochastic by nature. We compare our algorithm with two baselines: \emph{Oblivious-Q} and \emph{Augmented-Q}. Oblivious-Q is the standard Q-learning that ignores delay and assumes each decision to be immediately executed. Augmented-Q acts on the $m-$AMDP introduced in Def.~\ref{def: EDMDP}. We test all domains on delays $m\in \{0, 5, 15, 25\}$ with 5 seeds per each run. All results are summarized in Fig.~\ref{fig: exp barplots}, and are provided in more detail with std. in Appx.~\ref{sec: atari table}, Table~\ref{table: exp summary}. 

{\bf Tabular Maze Domain.} 
We begin with testing Delayed-Q on a Maze domain \citep{openai}[\url{tinyurl.com/y34tmfm9}]. It is based on  tabular Q-learning and enables us to study the merits of our method decoupled from the coming DDQN added complexities. Moreover, it conveys the exponential complexity of Augmented-Q. The forward-model we construct is naturally tabular as well: it predicts a state $s'$ according to the highest visitation frequency given $(s,a)$.
The objective in Maze is to find the shortest path from a start position to a goal state in a randomly-generated $N\times N$ maze. 
Reaching the goal yields a reward of $1,$ and $-1/(10N^2)$ per step otherwise. The maximal episode length is $10N^2$ steps, so the cumulative reward is in $[-1,1]$. We also create a Noisy Maze environment that perturbs each action w.p. $p \in [0, 0.5].$

Convergence plots are given in Fig.~\ref{fig: convergence}. Delayed-Q outperforms the rest for all delay values $m$, while Oblivious-Q fails in all runs for $m>0.$ Since the augmented state-space grows exponentially with $m$, Augmented-Q converges more slowly as $m$ increases. In fact, for $m>15$ the simulation fails to run due to memory incapacity for the Q-table; this explains its absence in Figs.~\ref{fig: convergence}-\ref{fig: exp barplots}. To confirm the exponential complexity growth of Augmented-Q and compare it with  Delayed-Q, we trained both agents  with increasing delay values, and reported the number of training episodes each one required before reaching a cumulative reward of $0.5$. Fig.~\ref{fig:exp vs linear growth} clearly demonstrates the exponential (resp. linear) dependence of Augmented-Q (resp. Delayed-Q) in the delay value. The linear dependence of Delayed-Q in $m$ is not surprising: Delayed-Q is algorithmically identical to Q-learning, except for the $m$-step forward-model calls and the replay buffer shift of $m$ samples. 
To further analyze its sensitivity to the state-space size, we ran tabular Delayed-Q on increasing maze sizes, for a fixed $m=5$. As Fig.~\ref{fig:varying_maze_size} shows, the performance drops exponentially, suggesting high sensitivity to the state-space size and highlighting one shortcoming of MBS \citep{walsh2009learning} (see Sec.~\ref{section:intro}).

\begin{figure*}[htp]
	\centering
	\subfigure[]{\label{fig:heatmap_maze}
	\includegraphics[width=0.21\linewidth,keepaspectratio]{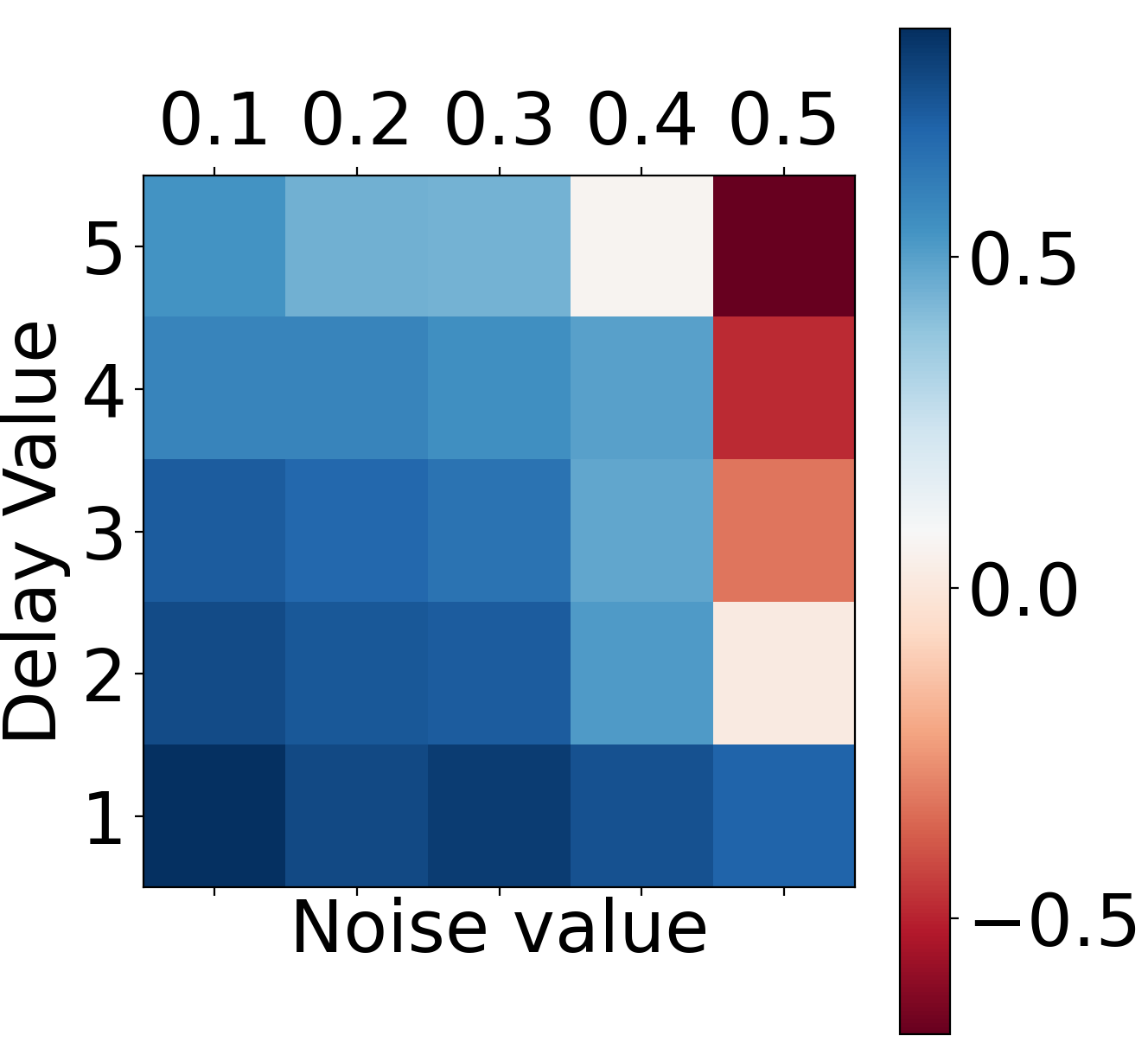}}
	\subfigure[]{\label{fig:heatmap_cartpole}
	\includegraphics[width=0.24\linewidth,keepaspectratio]{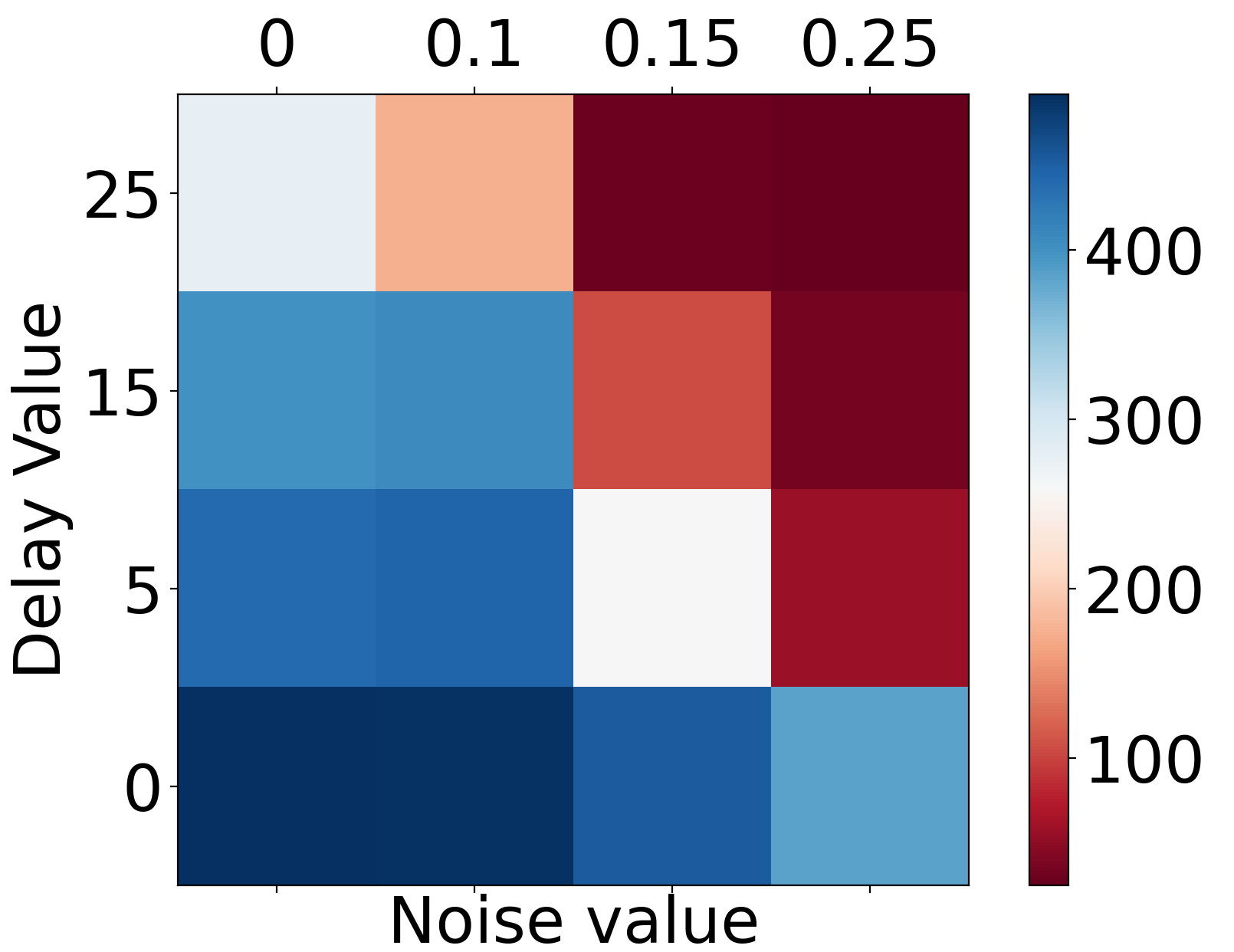}}
	\subfigure[]{\label{fig:varying_maze_size}
	\includegraphics[width=0.35\linewidth,keepaspectratio]{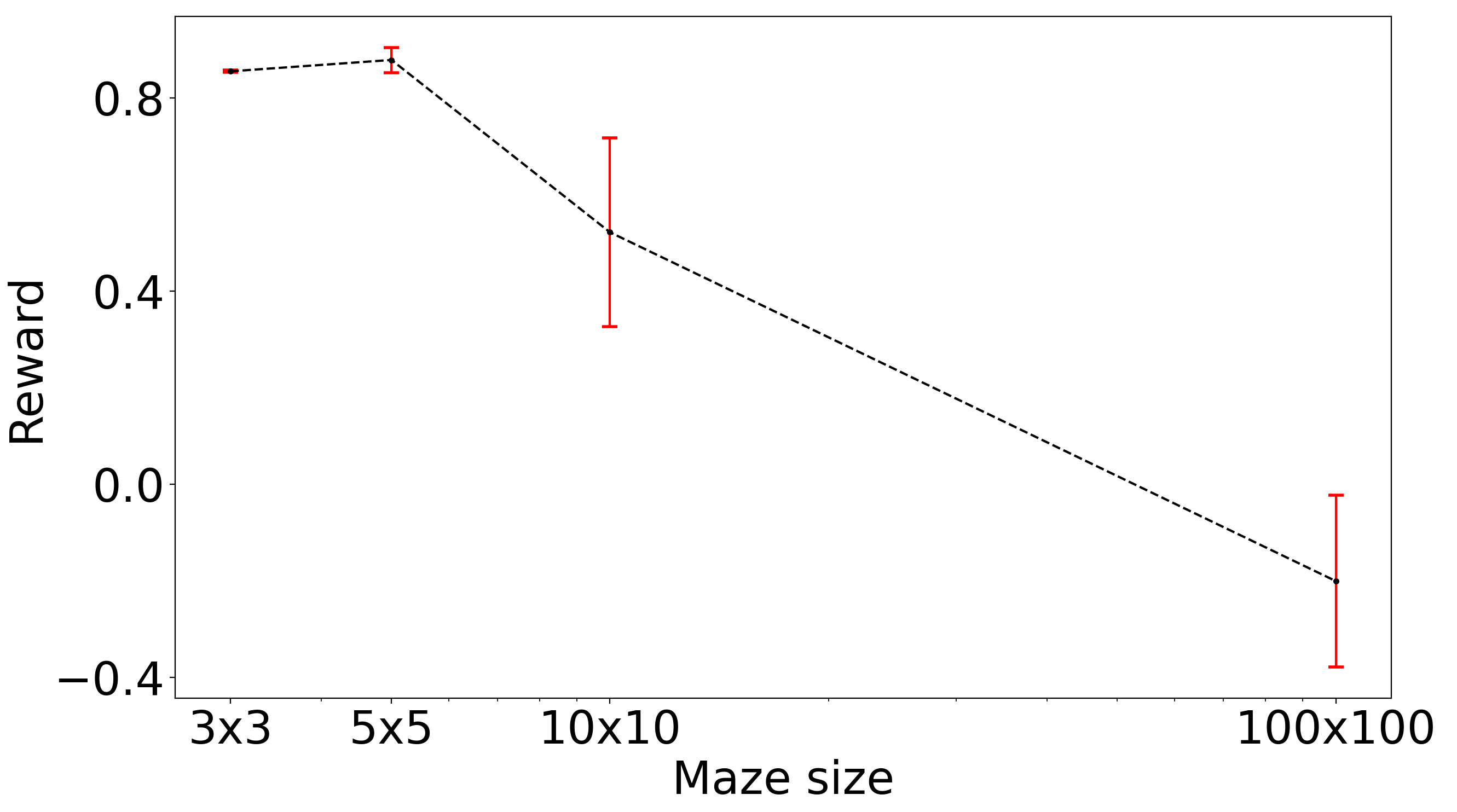}}
	\caption{Delayed-Q (median over 5 seeds): (a) Total reward after $5000$ training episodes on $10\times10$ Maze. Performance is sensitive to both delay value and stochasticity. (b) Noisy Cartpole. (c) Reward on varying Maze sizes. Abscissa is in log-scale, so the return decreases exponentially with $m$.} 
\end{figure*}

\begin{figure*}[!ht]
\begin{center}
\includegraphics[scale=0.33]{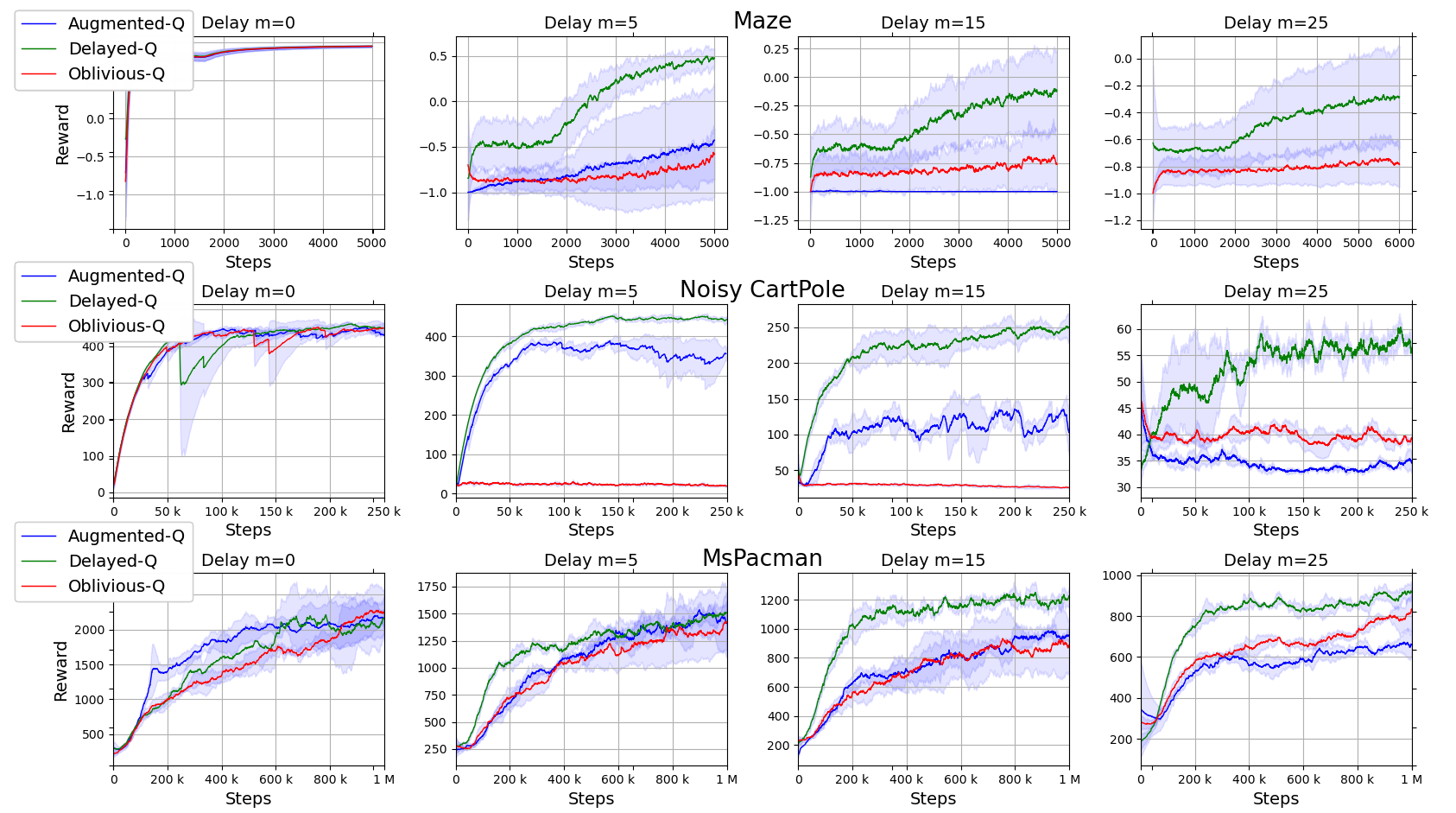}    
\end{center}
	\caption{Convergence plots for Maze, Noisy Cartpole and Atari MsPacman. Note that the scale of the y-axes (performance) may change from figure to figure.}
    \label{fig: convergence} 
\end{figure*}
\vspace{-0.2cm}
{\bf Physical Domains.} 
Next, we test our approach on two continuous domains: CartPole\footnote{Since Cartpole fails in $\sim$10 steps if the initial actions are random, we initialize the $m$-lengthed action-queue with optimal actions using a pretrained non-delayed model. We wait for $2m$ steps before starting to append samples to the replay buffer to avoid unfair advantage due to these actions.} and Acrobot. The CartPole task requires balancing a pole connected to a cart that actuates left or right. In Acrobot, one needs to swing up the lower of two links connected by a joint above a certain height. The agent receives a reward of $1$ if the pole stays above a certain angle in Cartpole, and in Acrobot it receives $-1$ until it reaches the goal. The episode length is 500 steps in both tasks. We also create noisy versions of both tasks: 
At each step, normal additive noises are independently added to each physical component's mass, with std of $0.1$ of the nominal mass.

We extend the famous DDQN algorithm \citep{van2015deep} and compare to it, though our method is general and can be seamlessly integrated into any Q-learning based algorithm. Our one-step forward-model is implemented with a neural network (NN) of the same architecture as the Q-network. Namely, it consists of two hidden layers, each of width 24, with ReLu activations. The input of the forward-model NN is the concatenation of $(s, a)$ and its output is $s'.$ Training the forward-model NN is conducted together with the Q-network training with the same hyperparameters and sample batches; this makes the implementation easy and simple. For Augmented-Q, a concatenation of the pending actions to the state is fed to the Q-network.
\begin{figure}
 \centering
    \includegraphics[scale=0.33]{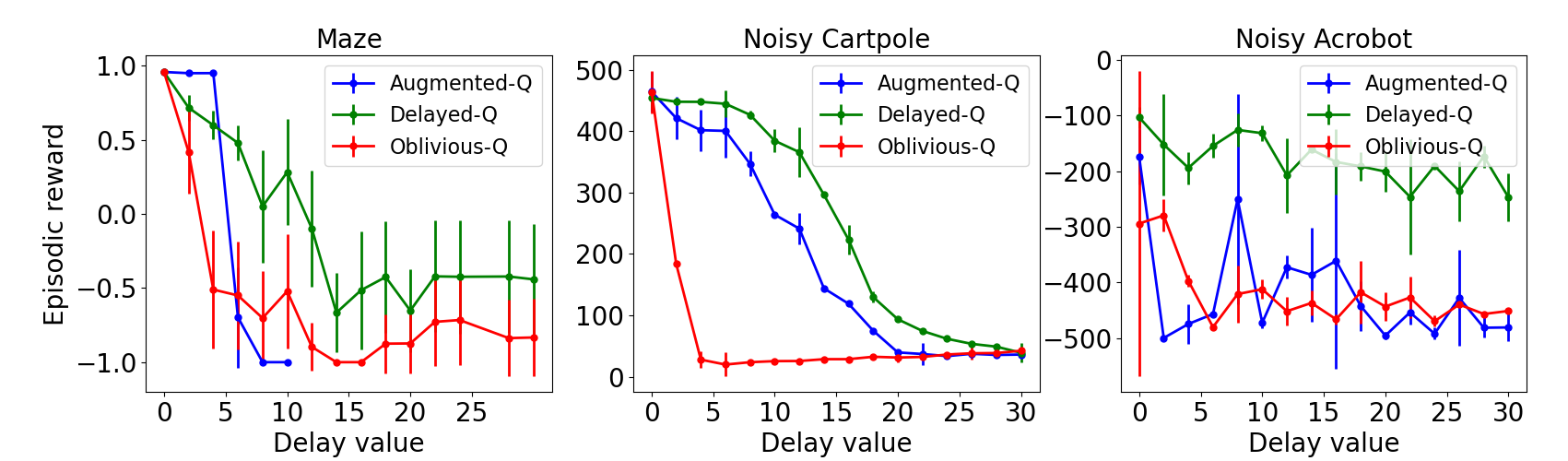}
    \captionof{figure}{Performance as a function of the delay (from left to right): Maze, Noisy Cartpole, Noisy Acrobot. For Augmented-Q in Maze, $m> 10$ is missing due to explosion of the state-space.}
    \label{fig:m sweep}
\end{figure}

\begin{wrapfigure}{R}{.37\textwidth}
\vspace{-0.35cm}
 \centering
    \includegraphics[scale=0.11]{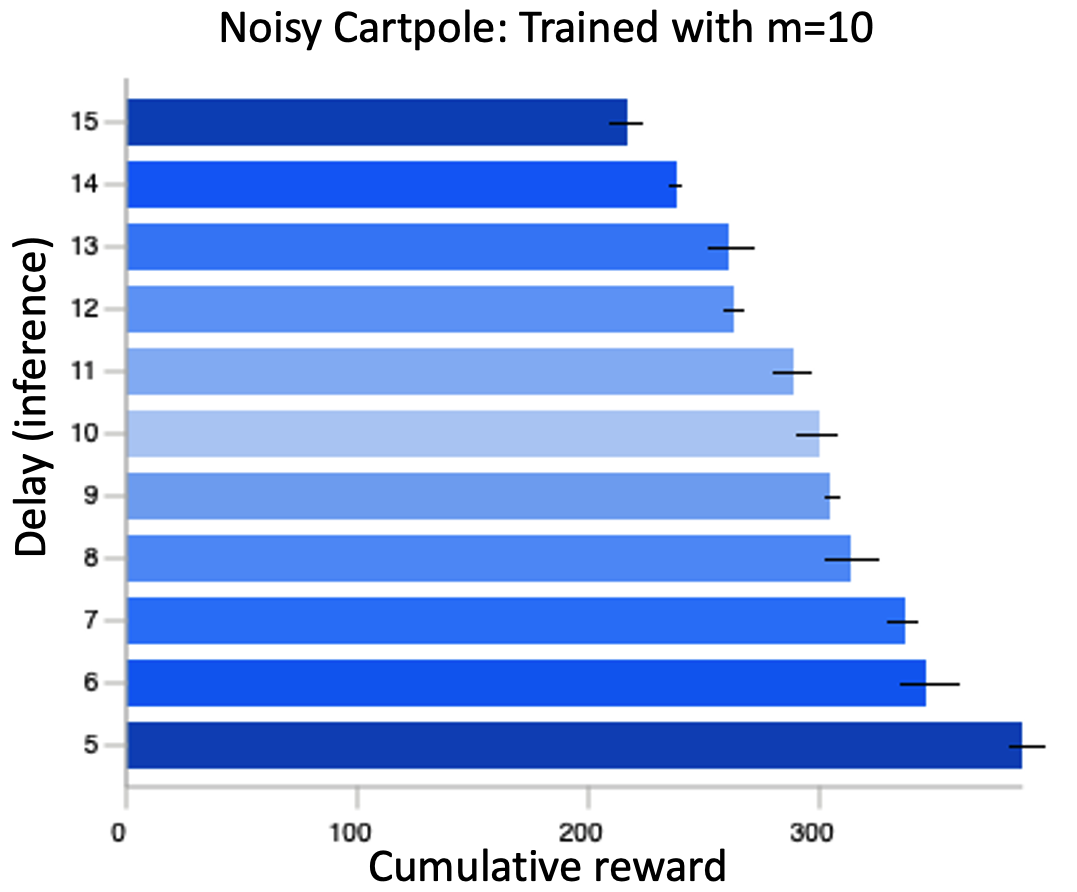}
    \captionof{figure}{Performance gap for Delayed-Q trained with a delay of $m=10$.}
    \label{fig:delay misspecification}
\end{wrapfigure}
Fig.~\ref{fig: convergence} depicts the performance of the three algorithms for different values of $m$ for Noisy Cartpole. As expected from a physical domain, ignoring delay gives catastrophic results even for $m=5.$ Augmented-Q performs moderately up to $m=15,$ but fails for larger delays. Delayed-Q performs the best for all $m$ values, and performs well even on the challenging task of balancing a noisy pole with $m=25.$ We observe similar behavior in all Cartpole and Acrobot experiments, as shown in Fig.~\ref{fig: exp barplots}. Moreover, in Fig.~\ref{fig:m sweep}, we demonstrate the relative robustness of Delayed-Q to different delay values. All tested environments exhibit superior performance of Delayed-Q for a wide range of delays. In Noisy Acrobot, Delayed-Q performs better for $m=25$ than the alternatives do for $m=2$. 
Figs.~\ref{fig:heatmap_maze}-\ref{fig:heatmap_cartpole} show a clear trade-off between noise and delay, as we also discuss in Rmk.~\ref{remark:noise delay tradeoff}. For high delays, the agent is much more sensitive to an increase in stochasticity. 

To quantify the dependence of Delayed-Q on the model accuracy, we compare the learned model to a perfect one, i.e., the environment itself. Fig.~\ref{fig: perfect learned gap} shows performance is impaired more as the delay increases and suggests a better model can potentially improve reward by 20-30\%.
\begin{figure}
    \centering
    \includegraphics[scale=0.23]{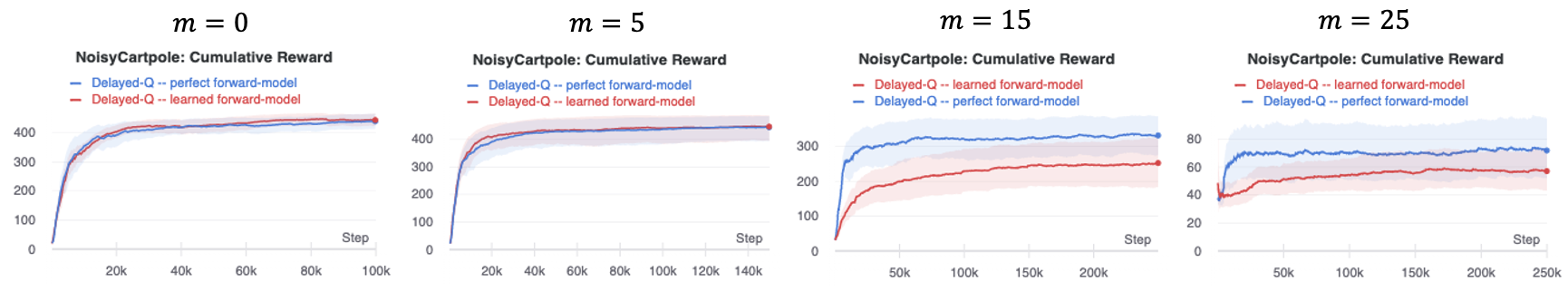}
    \caption{Noisy Cartpole: Performance gap between true and trained forward model. }
    \label{fig: perfect learned gap}
\end{figure}
Further, we test the robustness of Delayed-Q to misspecified delay by training it with $m=10$ and evaluating on other delay values. Fig.~\ref{fig:delay misspecification} shows the evaluation performance for $m \in \{5,\dots,15\}$. It demonstrates the robustness of our method -- varying performance in evaluation (for good or bad) does not stem from delay misspecification. Instead, the delay is ‘forgotten’ after training, and Fig.~\ref{fig:delay misspecification} depicts the general effect of execution delay on performance. For shorter delay than the training one, \ie $m < 10$, performance even improves. The reason is that, first, during training, the Q-function is ‘un-delayed’ due to the replay buffer shift that relates the actions to the correct execution time. Second, the forward-model is trained based on single-step transitions and only during inference is it queried $m$ times. Thus, these two networks composing the agent are oblivious to the delay they were trained on. 

{\bf Atari Domains.} 
We run the last set of experiments on the Atari Learning Environment \citep{bellemare2013arcade}. We inspect 8 games from those that were successfully tackled with the original Q-network architecture and hyperparameters of DDQN \citep{van2015deep}. Since a learned forward-model for images conditioned on actions is a hanging question in the research frontier, we leave it for future work and use the simulator itself for prediction. It is stochastic in nature and thus encompasses approximation error. For Augmented-Q, we concatenate the action queue to the output of the CNN part of the Q-network; the extended vector is then fed into the subsequent fully-connected part of it. We train all games for 1M steps.  
Fig.~\ref{fig: convergence} shows convergence plots for MsPacman. Delayed-Q is consistently better than Augmented-Q for all $m$ values, which is, in turn, better than Oblivious-Q. Although the gap between all three algorithms is small for $m=5,$ it increases with $m$. For $m=25,$ the delay is too large for the augmentation to have a positive effect compared to Oblivious-Q, and they perform the same. This behavior is representative of all Atari games, as can be seen in Fig.~\ref{fig: exp barplots}. 
Lastly, we compared Delayed-Q with a fourth algorithm which uses an RNN policy that is unaware of the delay value. The results are given in Appx.~\ref{sec: rnn comparison}, showing that a recurrent policy does not improve upon Augmented-Q or Oblivious-Q. This result is not surprising though: as stated in Thm.~\ref{theorem: markov policy is sufficient}, the history sequence $s_{t-m}, s_{t-m-1}, \dots$ does not aid the policy any further than only using $s_{t-m}$. 

\begin{figure}
    \centering
 \includegraphics[scale=0.45]{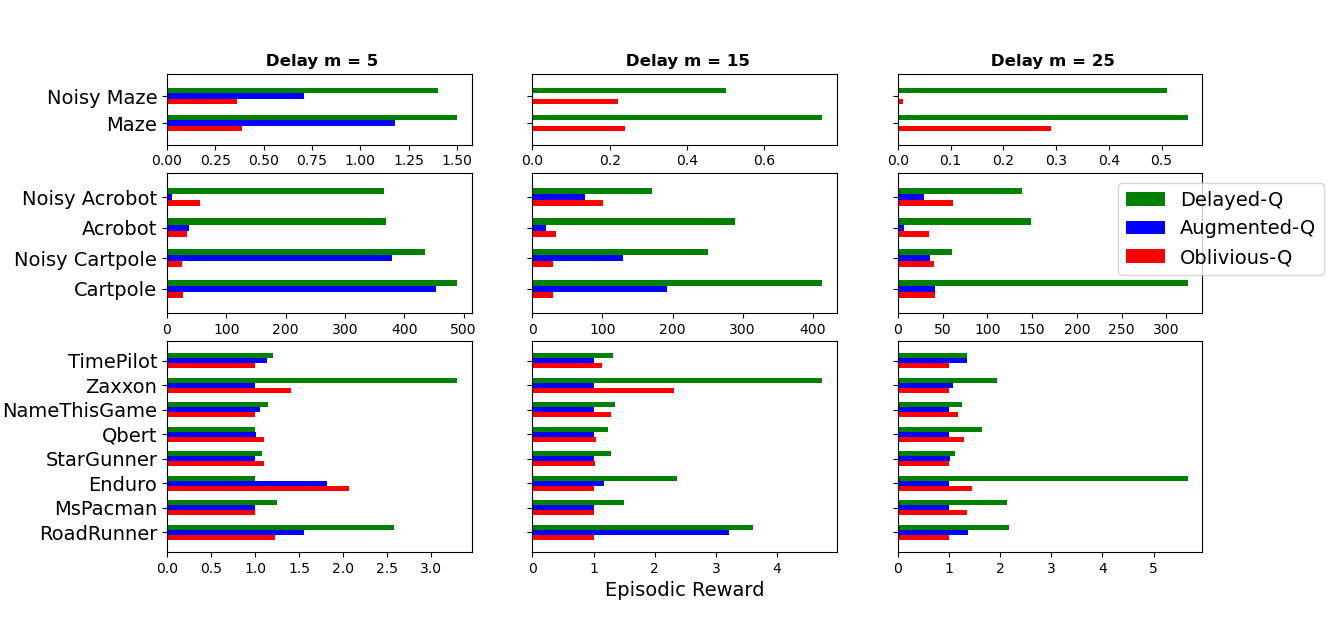}
    \caption{Experiment summary: mean of episodic return for all domains. Delayed-Q outperforms the alternatives in 39 of 42 experiments. Due to negative reward, a
    positive translation of 1 is applied for Maze and 500 for Acrobot. Atari x-axis is the gain relative to lowest result in each experiment. }
    \label{fig: exp barplots}
\end{figure}

 \vspace{-0.5cm}
\section{Discussion}
\vspace{-0.5cm}
In this work, we found that non-stationary deterministic Markov policies are optimal in delayed MDPs. 
Though more expressive, the standard state augmentation approach is intractable for all but the shortest delays, while the oblivious approach that ignores delay suffers from inferior performance.
We derived a Q-learning based algorithm that generates a Markov policy by combining a transition forward model with Q-network. 
The forward-model produces a simple future-state estimate. Incorporating probabilistic estimates and other improvements such as integration of image-based action-dependent learned forward-models \citep{kim2020learning}, are left for future research. 
Extensions of our work for real-world applications can be unknown or varying delay. In the first case, a good prior for the delay value can often be used, \eg for autonomous vehicles, as the latency statistics of the different hardware and software components are well studied \citep{zhao2019towards,niu201926ms}, while in production systems, they are almost constant \citep{toschi2019characterizing}. 
Our algorithm is also readily extendable to the second case of varying delay. 
Differently from the augmentation approach, our 1-step forward-model decouples the algorithm from the delay used for training, as Fig.~\ref{fig:delay misspecification} depicts. Also, quantization of the delay is not essential as long as the forward model can operate with variable delay values. 
Finally, our framework can be extended to policy-gradient-based methods that are particularly useful for continuous control, where observation delay is inherent. 

\section*{Acknowledgements}
The authors would like to thank Daniel J. Mankowitz and Timothy A. Mann for motivating this work.

\bibliography{lag_delay}
\bibliographystyle{iclr2021_conference}

\newpage
\appendix
\section*{Appendix}

%%%%%%%%%%%%%%%%%%%%%%%%%%%%%%%%%%%%%%%%%%%%%%%%%%%%%%%%%%%%%%%%%%
\section{MDPs with Delay: A Degradation Example}
%%%%%%%%%%%%%%%%%%%%%%%%%%%%%%%%%%%%%%%%%%%%%%%%%%%%%%%%%%%%%%%%%%
\subsection{Proof of Proposition~\ref{prop: degredation}}
\label{sec: example proof}
    Without loss of generality, assume $p\in[0.5, 1].$ It is easy to see than in our 2-state MDP, the optimal policy selects $a_0$ if the most likely state of the system is $s_0,$ and $a_1$ if it is $s_1.$ Since $p \geq0.5,$ the most-likely state of the system when observing $s_0$ is $s_0$ if $m$ is even, and $s_1$ if $m$ is odd. The same logic holds when observing $s_1.$ Therefore, if $m$ is even, $\pi^*(s_0)=a_0,\pi^*(s_1)=a_1.$ Otherwise, $\pi^*(s_0)=a_1,\pi^*(s_1)=a_0.$ Note that re-iterating the rest of the proof with a randomized policy (of the form $\pi(a|s) \in (0,1),~\forall a\in\{a_0,a_1\},~\forall s\in\{s_0,s_1\}$) yields sub-optimal return. Hence, in this example it is enough to consider deterministic policies.
   
    The expected reward at time $t+m$ with action $\pi^*(s_t)$ selected at $s_t$ is 
    \begin{equation}
    \label{eq: Rt+m}
    \begin{split}
        R^*_{t+m}(s_t) :&= \mathbb{E}_{s_{t+m}|s_t}[r(s_{t+m}, \pi^*(s_t))]\\
         &=  r(s_0, \pi^*(s_t))\mathbb{P}(s_{t+m}=s_0|s_t) + r(s_1, \pi^*(s_t))\mathbb{P}(s_{t+m}=s_1|s_t).
    \end{split}
    \end{equation}
    From here on, we inspect the case where $s_{t}=s_0$ for brevity. By symmetry, identical arguments apply if $s_{t}=s_1.$
    If $m$ is even, $r(s_0, \pi^*(s_0))=1$ and $r(s_1, \pi^*(s_0))=0.$ If $m$ is odd, $r(s_0, \pi^*(s_0))=0$ and $r(s_1, \pi^*(s_0))=1.$ Thus, using \eqref{eq: Rt+m},
 \begin{equation}
 \label{eq: R cases prob}
    R^*_{t+m}(s_0)  =     \begin{cases}
        \mathbb{P}(s_{t+m}=s_0|s_t=s_0) & \text{if $m$ is even,} \\
      \mathbb{P}(s_{t+m}=s_1|s_t=s_0) & \text{if $m$ is odd.}
        \end{cases}
\end{equation}
    Note that, by construction, the transition probabilities are independent of the actions. Specifically, if $m$ is even, 
    \begin{equation}
 \label{eq: m even}
    \mathbb{P}(s_{t+m}=s_0|s_t=s_0) =  \sum_{k ~ even}^m \binom{m}{k} p^k (1-p)^{m-k},
\end{equation}
    since we count the possibilities of an even number of jumps between the two states. Similarly, if $m$ is odd, 
  \begin{equation}
 \label{eq: m odd}
    \mathbb{P}(s_{t+m}=s_1|s_t=s_0) =  \sum_{k ~ odd}^m \binom{m}{k} p^k (1-p)^{m-k}.
\end{equation}
    Also note that the same applies for $s_t=s_1,$ i.e.,
    \begin{equation}
        R^*_{t+m}(s_0)=R^*_{t+m}(s_1)~ \quad\forall t, \label{eq: R constant s}
    \end{equation}
    and that these probabilities are independent of $t,$ i.e.,
        \begin{equation}
        R^*_{t+m}(s_0)=R^*_{t+m+k}(s_0)~\quad\forall k\in \mathbb{N}. \label{eq: R constant t}
    \end{equation}
    
    Next, we compute the optimal return starting from  $s_0:$ 
    \begin{align}
        v^*_m(s_0) :&= \mathbb{E}^{\pi^*}\left[\sum_{t=0}^\infty \gamma^t r(s_{t+m},\pi^*(s_{t}))|s_{t=0}=s_0\right] \nonumber \\
        &= R^*_m(s_0) + \gamma \left[\mathbb{P}(s_{t=1}=s_0)R^*_{m+1}(s_0) + \mathbb{P}(s_{t=1}=s_1))R^*_{m+1}(s_1) \right] \nonumber\\
        &\quad+ \gamma^2 \left(\mathbb{P}(s_{t=2}=s_0)R^*_{m+2}(s_0) + \mathbb{P}(s_{t=2}=s_1))R^*_{m+2}(s_1) \right)  + \dots \nonumber\\
        &= R^*_m(s_0) + \gamma R^*_{m+1}(s_0) \left(\mathbb{P}(s_{t=1}=s_0) +\mathbb{P}(s_{t=1}=s_1))\right) \nonumber\\
        &\quad+
        \gamma^2 R^*_{m+2}(s_0)\left(\mathbb{P}(s_{t=2}=s_0) +\mathbb{P}(s_{t=2}=s_1))\right) +\dots \nonumber\\
        &= \frac{1}{1-\gamma} R^*_m(s_0), \label{eq: R summation}
    \end{align}
    where in the second relation we used \eqref{eq: R constant s}, and in the last relation \eqref{eq: R constant t} as well as $\mathbb{P}(s_t=s_0) = 1 - \mathbb{P}(s_t=s_1) ~ \quad\forall t.$
    
    Plugging \eqref{eq: m even} and \eqref{eq: m odd} into \eqref{eq: R cases prob}, together with \eqref{eq: R constant t}, \eqref{eq: R summation} and \eqref{eq: R constant s} gives the optimal return 
    \begin{equation}
    \label{eq: first part result}
        v^*_m(s_0) = v^*_m(s_1) = \begin{cases}
      \frac{1}{1 - \gamma}  \sum_{k ~ even}^m \binom{m}{k} p^k (1-p)^{m-k}, & \text{if $m$ is even,} \\
      \frac{1}{1 - \gamma}  \sum_{k ~ odd}^m \binom{m}{k} p^k (1-p)^{m-k},  & \text{if $m$ is odd.} 
        \end{cases}
    \end{equation}
    This concludes the first part of the proof. 
    
    In the second part, we shall now derive a simpler expression for \eqref{eq: first part result} which can then be analyzed to determine monotonicity w.r.t. $m$ and $p.$ Observe that
    \begin{align*}
        (1-2p)^m &= (-p + 1-p)^m =\sum_{k}^{m} {m\choose k} (-p)^k (1-p)^{m-k} \\
        &= \sum_{k \ even}^{m} {m\choose k} p^k (1-p)^{m-k} - \sum_{k \ odd}^m {m\choose k} p^k (1-p)^{m-k}.
    \end{align*}    
    Since
    $$
    \sum_{k \ even}^{m} {m\choose k} p^k (1-p)^{m-k} + \sum_{k \ odd}^m {m\choose k} p^k (1-p)^{m-k} = 1,
    $$
    we have that
    \begin{align}
    \sum_{k \ even}^{m} {m\choose k} p^k (1-p)^{m-k}  = \frac{1}{2} \left( 1 + (1-2p)^m \right), \\
    \sum_{k \ odd}^{m} {m\choose k} p^k (1-p)^{m-k} = \frac{1}{2} \left( 1 - (1-2p)^m \right).
    \end{align}
    Denote $a:=-(1-2p),$ remember that $0 \leq a \leq 1,$ and let $m=2n$ (resp. $m=2n+1$) with $n\in \mathbb{N}$ when $m$ is even (resp. odd).
    Then 
    \begin{equation}
        \label{eq: writing even with a}
        \frac{1}{2} \left( 1 + (1-2p)^m \right) = \frac{1}{2} \left( 1 + (a^2)^n \right)
    \end{equation}
    and 
    \begin{equation}
        \label{eq: writing odd with a}
        \frac{1}{2} \left( 1 - (1-2p)^m \right) = \frac{1}{2} \left( 1 + a (a^2)^n \right).
    \end{equation}
    Both \eqref{eq: writing even with a} and \eqref{eq: writing odd with a} obviously  monotonically decrease with $n,$ so the even and odd subsequences are monotone. Also, since $a\leq1,$ \eqref{eq: writing odd with a} $\leq$ \eqref{eq: writing even with a}, which gives that the whole sequence itself is monotone in $m$.
    Lastly, as $p$ increases $a$ increases. This obviously causes both \eqref{eq: writing even with a} and \eqref{eq: writing odd with a} to increase as well.

%%%%%%%%%%%%%%%%%%%%%%%%%%%%%%%%%%%%%%%%%%%%%%%%%%%%%%%%%%%%%%%%%%
\section{The Standard Approach: Augmentation}
%%%%%%%%%%%%%%%%%%%%%%%%%%%%%%%%%%%%%%%%%%%%%%%%%%%%%%%%%%%%%%%%%%
\subsection{The augmented MDP}
\label{appendix: augmented mdp construction}
	Let the augmented state space $\X_m := \St\times\A^m.$  Then, $x_t := (s_t, a_{t}^{-1}, \cdots, a_{t}^{-m})\in\X_m$ is an extended state,
where $a_{t}^{-i}$ is the $i$-th pending action at time $t$. It means that in the following step, $t+1$, action $a_t^{-m}$ will be executed independently of the present action selection. 
Accordingly, a new transition function for $\X_m$ is induced by the original transition matrix $P$ and $m$-step delay. More explicitly, for $(x,a,x')\in\X_m\times\A\times\X_m$ we have
\begin{equation}
\label{eq:transition_F}
\begin{aligned}
    F(x' | x, a)
    &=  \begin{cases}
        P(e_1^\top x' | e_1^\top x, e_{m+1}^\top x) & \text{if }  e_2^\top x' = a \text{ and } e_{i+1}^\top x' = e_i^\top x ~
        \forall i\in[2:m], \\
        0 & \mbox{otherwise},
        \end{cases}
\end{aligned}
\end{equation}
where $e_i\in \{0,1\}^{m+1}$ is the elementary vector with $1$ only in its $i$-th coordinate.\footnote{Throughout this work, we assume without loss of generality that $s \in \St$ is a scalar, to simplify notation of inner products with $e_i$. This assumption is non-limiting since any multi-dimensional state space can be easily transformed to single-dimensional via enumeration as it is finite.} Similarly, the reward function on the augmented state-space is:
\begin{equation}
\label{eq:reward_function}
    \begin{split}
  g(x,a) = r(e_1^\top x, e_{m+1}^\top x).
    \end{split}
\end{equation}
Note that $g$ does not depend on the newly decided action $a\in\A$, but rather on the first and last coordinates of the current state $x\in\X_m$. This leads us to the following definition.

\subsection{$m$A-PI Algorithm}
Let the set of greedy policies w.r.t. $v\in\mathbb{R}^{\abs{\X_m}}$: 
$
\bar{\mathcal{G}}(v) := \{\piD\in \PiD: \bar{T}^{\piD}v = \bar{T}v\}.
$
\label{sec: med_pi_algo}
 \begin{algorithm}[H]
  \centering
\caption{$m$A-PI}
  \label{alg: m-execution_delay_PI}
\begin{algorithmic}[1]
  \STATE {\bfseries Initialize:} $\piD_0 \in \PiD, k = 0$
  \WHILE {$\piD_k$ is changing}
  \STATE $v_{k} \leftarrow v^{\piD_k}$
  \STATE $\piD_{k+1} \leftarrow $ any element of $\bar{\mathcal{G}}(v_k)$
  \STATE $k\leftarrow k+1$
  \ENDWHILE
  \STATE {\bfseries Return: } $\piD_k, v_k$ 
\end{algorithmic}
\end{algorithm}

\subsection{Convergence of $m$A-PI}
\label{sec: bellman_proof}

Convergence of $m$A-PI directly follows from the improvement property of greedy policies, which we prove below. 

\begin{proposition*}[$m$A Evaluation and Improvement]
\label{proposition:bellmanED_evaluation_appendix}
\begin{itemize*}
    \item[(i)] For any $x\in\X_m$ and $\piD \in \PiD$, the augmented value function $v^{\piD}$ satisfies the Bellman recursion $v^{\piD}(x) = \bar{T}^{\piD}v^{\piD}(x)$.\\
    
    \item[(ii)] The optimal augmented value $\bar{v}^*$ is the unique fixed point of $\bar{T}.$ Furthermore, if $\piD^*$ is preserving, \ie $\piD^* \in\arg\max_{\piD} \left\{g^{\piD} + \gamma F^{\piD} \bar{v}^*\right\}$, then $\piD^*$ is optimal and thus, $\bar{v}^* = \bar{v}^{\piD^*}$.
\end{itemize*}
\end{proposition*}

\begin{proof}
Using standard Bellman recursion on the augmented MDP, we can write
\begin{equation*}
\begin{split}
v^{\piD}(x)&= \mathbb{E}^{\piD}\left[ \sum_{t = 0}^{\infty}\gamma^{t}g(x_t, a_t) | x_0 = x\right] \\
&= \mathbb{E}^{\piD}\left[g(x_0, a_0) +  \sum_{t=1}^{\infty}\gamma^{t} g(x_t, a_t)| x_0 = x\right]\\
    &= g(e_1^\top x, e_{m+1}^\top x) + \gamma \mathbb{E}^{\piD}\left[ \sum_{t = 0}^{\infty}\gamma^{t} g(x_{t+1}, a_{t+1}) | x_0 = x\right] \\
    &= g(e_1^\top x, e_{m+1}^\top x)  + \gamma \sum_{(x',a)\in\X_m\times\A} \piD(a| x)F(x' | x, a)v^{\piD}(x')\\
    &= \bar{T}^{\piD}v^{\piD}(x)
\end{split}
\end{equation*}
which ends the proof of Claim (i). 

Note that by definition of $g$ and $F$ as in Equations~\eqref{eq:reward_function} and \eqref{eq:transition_F} respectively, the sum can be reformulated as follows:
\begin{equation*}
\begin{split}
v^{\piD}(x)&=r(e_1^\top x, e_{m+1}^\top x) + \gamma
    \sum\limits_{\substack{(x',a)\in\X_m\times\A:\\ e_{i+1}^\top x' = e_i^\top x \text{ for } i\in[2:m];\\ e_2^\top x' = a }}  \piD(a| x)P(e_1^\top x' | e_1^\top x, e_{m+1}^\top x)v^{\piD}(x')\\
    &=r(e_1^\top  x, e_{m+1}^\top x) + \gamma\sum_{(s',a)\in\St\times\A}\piD(a| x) P(s' | e_1^\top x, e_{m+1}^\top x)v^{\piD}(s', a, e_2^\top  x,\cdots, e_m^\top x)
\end{split}
\end{equation*}
Claim (ii) relies on classical theory of discounted MDPs, applied to the augmented MDP \citep{puterman2014markov}.
\end{proof}

\subsection{Proof of Theorem~\ref{theorem: maPI lower bound}}
\label{sec: general lower bound}

First, we give a general lower bound to the classic PI algorithm by Howard \citep{howard1960dynamic} for non-delayed MDPs, that immediately confirms the exponential complexity of $m$A-PI. 
\begin{proposition*}[Lower Bound for Howard's PI]
The number of iterations required for Howard's PI to converge in standard MDP $(\St, \A, P, r, \gamma)$ is $\Omega(|\St|).$
\end{proposition*}

\begin{proof}
		To prove the lower bound, we construct an example infinite-horizon MDP in which Howard's PI updates exactly one state at each iteration, and the number of updates is $|\St| -1.$ 

		 The example MDP is given in Fig.~\ref{fig: mdp example}. It contains a row of $n + 1$  states $(s_0, s_1, \dots, s_n),$ and a single absorbing state $s_{n+1}$. The transitions are deterministic. From each state except for $s_{n+1}$ there are two actions, $u$ and $d$, which respectively lead to the next state in the sequence or to $s_{n+1}$. The last state in the row, $s_n,$ leads to itself or $s_{n+1}$ by respectively taking actions $u$ or $d.$ Any action leads $s_{n+1}$ to itself.  The rewards are $0$ everywhere except for $r(s_n,u) := 1 - \gamma.$ We denote by $(v_t ,\pi_{t + 1})$ the value-policy pair at iteration $t$ of Howard's PI. We shall now describe the convergence process to the optimal policy, which is obviously $\pi^*(s) = u ~\forall s \in \St \setminus \{s_{n+1}\}$. 
		
		\textit{Initialization:} Set $\pi_0(s) = d ~\forall s \in \St \setminus \{s_{n+1}\}.$
		
		\textit{Iteration $0$:}  Clearly, $v_0 = \mathbf{0}.$ Then, for all  $s \in \St \setminus \{s_n, s_{n+1}\}, ~ \pi_1(s) = \arg\max_a \{r(s, a) + \gamma v_0(s')\} = \arg\max_a \{0, 0\} =d$ ~\footnote{The policy improvement step needs to choose between two actions that both yield values $0$. Without loss of generality, in such case, it simply chooses according to the lowest index, giving $d$ here.}. Also, $\pi_1(s_n) = u$ since $1-\gamma > 0.$
		
		\textit{Iteration $t~(t=1,\dots,n)$:} We have $${v_t(s_i) = 0 ~ \text{ for } i \in \{0, \dots, n-t\}}$$ and  $${v_t(s_i) = \gamma^{n-i} \frac{1-\gamma}{1-\gamma} = \gamma^{n-i} ~ \text{ for } i \in \{n-t+1,\dots,n\}}.$$ The policy output is thus $$\pi_{t+1}(s_i) = d ~ \text{ for } i \in \{0,\dots n-t-1\}$$ and $$\pi_{t+1}(s_i) = u ~ \text{ for } i \in \{n-t,\dots n\}.$$
		
		To summarize, at each iteration a single state updates its action to the optimal one such that at iteration $t,$ the policy stabilizes on $\pi(s)=u$ for all $s \in \{s_{n-t},\dots,s_n\}.$ Therefore, the total number of iterations until convergence is $n + 1 = |\St| - 1.$
		
		\begin{figure}
        \centering
        \includegraphics[scale=0.65]{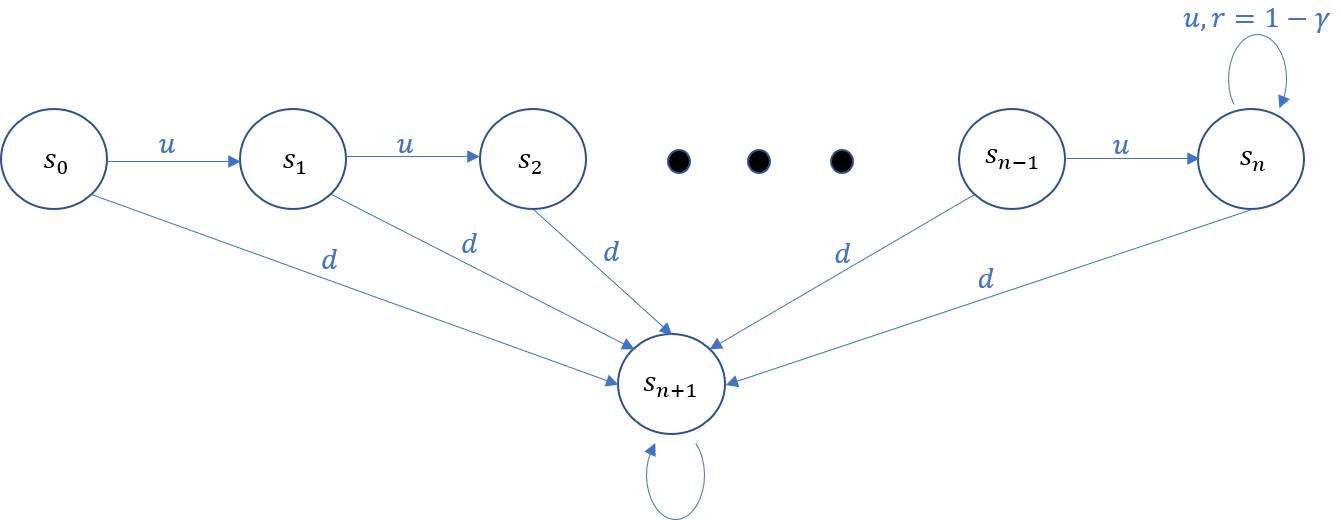}
        \caption{MDP example:  the transitions are deterministic and the rewards are $0$ everywhere except for $r(s_n,u) = 1 - \gamma.$}
        \label{fig: mdp example}
    \end{figure}
	\end{proof}
	
The exponential complexity of $m$A-PI follows, as stated in Thm.~\ref{theorem: maPI lower bound} that we recall below:
\begin{proposition*}[Lower Bound for $m$A-PI]
The number of iterations required for $m$A-PI 
to converge in  $m$-EDMDP $\mathcal{M}_m$ is $\Omega(|\X_m|) = \Omega(|\St||\A|^m).$
\end{proposition*}

\subsection{Proof of Theorem \ref{thm:delay_PI_augmented}}
\label{sec: augmented upper bound}
\begin{theorem*}[$m$A-PI Convergence]
The $m$A-PI algorithm as given in Alg.\ref{alg: m-execution_delay_PI} converges to the optimal value-policy pair $(\bar{v}^*, \piD^*)$ in at most $$\abs{\St}\abs{\A}^m(\abs{\A} - 1)\ceil*{\log\left(\frac{1}{\gamma}\right)^{-1}\log\left(\frac{1}{1-\gamma}\right)}$$ 
iterations. 
\end{theorem*}

\begin{proof}
The proof proceeds in three steps which follow the same lines as in \citep{scherrer2016improved} except that here, we adapt that method to the augmented MDP $\mathcal{M}_m$ with its corresponding Bellman operators $\bar{T}^{\piD}$ and $\bar{T}$ instead. For completeness, we recall these three steps whose proofs can be found in \citep{scherrer2016improved}.

Given policy $\piD_t$ output at iteration $t$, define the advantage of $\piD'$ w.r.t. $\piD$ as:
$$
a^{\piD'}_{\piD} := \bar{T}^{\piD'}v^{\piD} - v^{\piD}
$$
and the maximal advantage w.r.t. $\piD$ as 
$$
a_{\piD} := \max_{\piD'\in\PiD}a^{\piD'}_{\piD} = \max_{\piD'\in\PiD}\bar{T}^{\piD'}v^{\piD} - v^{\piD} = \bar{T}v^{\piD} - v^{\piD}.
$$

\textit{Step 1 \citep{scherrer2016improved}[Lemma 10]. }
For all augmented policies $\piD, \piD'\in\PiD$, 
$
v^{\piD'} - v^{\piD} = (\bar{I} - \gamma F^{\piD'})^{-1}a_{\piD}^{\piD'} = (\bar{I} - \gamma F^{\piD})^{-1} (-a_{\piD'}^{\piD}),
$
with $\bar{I}$ being the identity matrix in $\mathbb{R}^{|\X_m|\times|\A|}$.

\textit{Step 2 \citep{scherrer2016improved}[Lemma 2]. } 
Define as $\bar{v}^* = v^{\piD^*}$ the optimal value function of the augmented MDP $\mathcal{M}_m$ as defined in Def.~\ref{def: EDMDP}. Then, the sequence $(\norm{\bar{v}^* - v^{\piD_t}}_{\infty})_{t\geq 0}$ built by the $m$A-PI algorithm as given in Alg.\ref{alg: m-execution_delay_PI} is a $\gamma$-contraction w.r.t. the max-norm. 

\textit{Step 3 \citep{scherrer2016improved}[Section 7]. }
Let $x_0 \in\X_m$ be such that $-a_{\piD^*}^{\piD_0}(x_0) = \norm{a_{\piD^*}^{\piD_0}}_\infty$. Then, for all $t\geq 0$ we have
$$
-a_{\piD^*}^{\piD_t}(x_0) \leq \norm{a_{\piD^*}^{\piD_t}}_\infty \leq \frac{\gamma^t}{1-\gamma}\norm{a_{\piD^*}^{\piD_0}}_\infty = \frac{\gamma^t}{1-\gamma}(-a_{\piD^*}^{\piD_0}(x_0)).
$$

From there it results that $\piD_t(x_0)$ must be different from $\piD_0(x_0)$ whenever $\frac{\gamma^t}{1-\gamma} <1$, that is, for all iterations 
$$
t > \ceil*{\frac{\log(1/(1-\gamma))}{\log(1/\gamma)}}=: t^*.
$$
Therefore, one sub-optimal action is eliminated in favor of a better one within $t^*$ iterations. There are at most $\abs{\X_m}(\abs{\A}-1)$ of them, which ends the proof. 
\end{proof}

\section{Execution-Delay MDP: A New Formulation}
\label{sec: EDMDP formulation}
Let $\mu$ be the initial state distribution. Then policy $\pi\in\Pi^{\textsc{HR}}$ induces a probability measure on $(\Omega, \mathcal{B}(\Omega))$ denoted by $\mathbb{P}^{\pi}_m$ and defined through the following:
\begin{align}
    &\mathbb{P}^{\pi}_m(\tilde{s}_0 = s_0) = \mu (s_0);\label{eq: init state distribution}\\
    &\mathbb{P}^{\pi}_m(\tilde{a}_t = a | \tilde{h}_{t} = h_{t}) = \delta_{\bar{a}_t}(a), \qquad \forall t<m; \label{eq: init policy}\\
    &\mathbb{P}^{\pi}_m(\tilde{a}_t = a | \tilde{h}_{t-m} = h_{t-m}) = q_{d_{t-m}(h_{t-m})}(a), \qquad\forall t\geq m; \label{eq: delayed policy} \\
    &\mathbb{P}^{\pi}_m(\tilde{s}_{t+1} = s | \tilde{h}_t = (h_{t-1}, a_{t-1}, s_t), \tilde{a}_t = a_t) = P(s|s_t, a_t). \label{eq: state transition function}
\end{align}
\subsection{Proof of Proposition \ref{proposition: delayed process distribution}}
\label{appendix: delayed process distrib}
\begin{proof}
We first state the following, which holds by definition of conditional probability. For all measurable sets $A_1,\cdots, A_n \in \mathcal{B}(\Omega)$, we have 
\begin{equation}
\label{eq: definition conditional}
  \mathbb{P}_m^{\pi}(\cap_{i=1}^{n}A_i) = \left( \prod_{i=1}^{n-1}\mathbb{P}_m^{\pi}(A_i| \cap_{j=i+1}^{n} A_j)\right)\mathbb{P}_m^{\pi}(A_n).  
\end{equation}
Applying \eqref{eq: definition conditional} to $n = 2t+1$ on the following events:
\begin{align*}
    A_{2t+1} &:= \{\tilde{s}_0 = s_0\}\\
    A_{2t} &:= \{\tilde{a}_0 = a_0\}\\
    &\vdots \\
    A_2 &:= \{\tilde{a}_{t-1} = a_{t-1}\}\\
    A_1 &:= \{\tilde{s}_t = s_t\},
\end{align*}
we obtain that 
\begin{align*}
  &\mathbb{P}^{\pi}_m(\tilde{s}_0 = s_0, \tilde{a}_0 = a_0, \cdots, \tilde{a}_{t-1}  = a_{t-1}, \tilde{s}_t  = s_t)\\
  &= \mathbb{P}^{\pi}_m(\tilde{s}_0 = s_0) \prod_{i = 0}^{t-1}\mathbb{P}^{\pi}_m(\tilde{a}_i = a_i|\tilde{s}_0 = s_0, \tilde{a}_0 = a_0, \cdots, \tilde{s}_{i}  = s_{i})  \mathbb{P}^{\pi}_m(\tilde{s}_{i+1} = s_{i+1}|\tilde{s}_0 = s_0, \tilde{a}_0 = a_0, \cdots, \tilde{a}_{i}  = a_{i})\\
  &= \mathbb{P}^{\pi}_m(\tilde{s}_0 = s_0) \prod_{i = 0}^{t-1}\mathbb{P}^{\pi}_m(\tilde{a}_i = a_i|\tilde{h}_{i} = h_{i})
 \mathbb{P}^{\pi}_m(\tilde{s}_{i+1} = s_{i+1}|\tilde{h}_{i} = (h_{i-1}, a_{i-1}, s_i), \tilde{a}_{i}  = a_{i})
\end{align*}
If $t \leq m$, then $0\leq i < m$ and by Eqs.~\eqref{eq: init state distribution}, \eqref{eq: init policy} and \eqref{eq: state transition function},   
\begin{align*}
  \mathbb{P}^{\pi}_m(\tilde{s}_0 = s_0, \tilde{a}_0 = a_0, \cdots, \tilde{a}_{t-1}  = a_{t-1}, \tilde{s}_t  = s_t)
  =   \mu(s_0) \left(\prod_{i = 0}^{t-1}\delta_{\bar{a}_i}(a_i)
  P( s_{i+1}| s_i,  a_i)\right).
\end{align*}
Otherwise, by Eq.~\eqref{eq: delayed policy}, 
\begin{align*}
 &\mathbb{P}^{\pi}_m(\tilde{s}_0 = s_0, \tilde{a}_0 = a_0, \cdots, \tilde{a}_{t-1}  = a_{t-1}, \tilde{s}_t  = s_t)\\
  &=  \mathbb{P}^{\pi}_m(\tilde{s}_0 = s_0) \prod_{i = 0}^{m-1}\mathbb{P}^{\pi}_m(\tilde{a}_i = a_i|\tilde{h}_{i} = h_{i})
  \mathbb{P}^{\pi}_m(\tilde{s}_{i+1} = s_{i+1}|\tilde{h}_i = (h_{i-1}, a_{i-1}, s_i), \tilde{a}_i  = a_i)\\
  &\quad \prod_{k = m}^{t-1}\mathbb{P}^{\pi}_m(\tilde{a}_k = a_k|\tilde{h}_k = h_k)
 \mathbb{P}^{\pi}_m(\tilde{s}_{k+1} = s_{k+1}|\tilde{h}_k = (h_{k-1}, a_{k-1}, s_k), \tilde{a}_k  = a_k) \\
  &= \mu(s_0)\left(\prod_{i = 0}^{m-1}\delta_{\bar{a}_i}(a_i) P( s_{i+1}| s_i,  a_{i})\right)\left( \prod_{k = m}^{t-1}q_{d_{k-m}(h_{k-m})}(a_k) P( s_{k+1}| s_k,  a_{k})\right),
\end{align*}
which concludes the proof.
\end{proof}

\subsection{Remark regarding the Markov property}
\label{appx: remark markov property}
For $T > t \geq m$, the conditional probability can be evaluated through:
\begin{align*}
&\mathbb{P}^{\pi}_m(\tilde{a}_t = a_t, \tilde{s}_{t+1} = s_{t+1},\cdots, \tilde{a}_{T-1} = a_{T-1}, \tilde{s}_T = s_T| \tilde{s}_0 = s_0, \tilde{a}_0 = a_0,\cdots, \tilde{a}_{t-1} = a_{t-1},\tilde{s}_t  = s_t)\\
&= \frac{ \mathbb{P}^{\pi}_m(\tilde{s}_0 = s_0, \tilde{a}_0 = a_0, \cdots,\tilde{a}_{T-1} = a_{T-1}, \tilde{s}_T = s_T)}{ \mathbb{P}^{\pi}_m(\tilde{s}_0 = s_0, \tilde{a}_0 = a_0,\cdots, \tilde{a}_{t-1} = a_{t-1},\tilde{s}_t  = s_t)}\\
&= q_{d_{t-m}(h_{t-m})}(a_{t})
 P(s_{t+1}| s_{t}, a_{t}) \cdots q_{d_{T-m-1}(h_{T-m-1})}(a_{T-1}) P(s_T| s_{T-1}, a_{T-1}).
\end{align*}
For a stationary policy $\pi:=(d,d,\cdots)\in\Pi^{\textsc{SR}}$, this simplifies to  
\begin{align*}
&\mathbb{P}^{\pi}_m(\tilde{a}_t = a_t, \tilde{s}_{t+1} = s_{t+1},\cdots, \tilde{a}_{T-1} = a_{T-1}, \tilde{s}_T = s_T| \tilde{s}_0 = s_0, \tilde{a}_0 = a_0,\cdots, \tilde{a}_{t-1} = a_{t-1},\tilde{s}_t  = s_t)\\
&= q_{d(s_{t-m})}(a_{t})
 P(s_{t+1}| s_{t}, a_{t}) \cdots q_{d(s_{T-m-1})}(a_{T-1}) P(s_T| s_{T-1}, a_{T-1}).
\end{align*}
Observing that the resulting conditional probability is a function of past observations when $m>0$, we conclude that even under a stationary policy, the induced stochastic process is not a Markov chain. This is different from the standard MDP setting in which any Markov policy induces a discrete time Markov chain \citep{puterman2014markov}[Sec. 2.1.6].

\subsection{Proof of Theorem \ref{theorem: markov policy is sufficient}}
\label{appendix: markov policy is sufficient}
We first prove the following lemma, which will be used in the theorem's proof. 

\begin{lemma}
\label{lemma: independency}
For all $m > 0$, $t\geq 0$, 
\begin{align}
    \label{eq: rm}
    \mathbb{P}^{\pi}_m(\tilde{s}_{t+1} = s' | \tilde{a}_{t+1} = a', \tilde{s}_{t} = s, \tilde{a}_{t} = a) =
     \mathbb{P}^{\pi}_m(\tilde{s}_{t+1} = s' |  \tilde{s}_{t} = s, \tilde{a}_{t} = a)
\end{align}
\end{lemma}
\begin{proof}
First, note that for all delay value $m>0$, $\tilde{a}_{t+1}$ only depends on the history up to $t-m+1$, which is $h_{t-m+1} = (h_{t-m}, a_{t-m}, s_{t-m+1}),$ as Eq.~\eqref{eq: delayed policy} suggests. Thus, since $t-m+1< t+1$, we have that $\tilde{a}_{t+1}$ is independent of $\tilde{s}_{t+1}$.
Using Bayes rule, it follows that 
\begin{align*}
    &\mathbb{P}^{\pi}_m(\tilde{s}_{t+1} = s' | \tilde{a}_{t+1} = a', \tilde{s}_{t} = s, \tilde{a}_{t} = a) \\
    &=\frac{\mathbb{P}^{\pi}_m(\tilde{a}_{t+1} = a' | \tilde{s}_{t+1} = s', \tilde{s}_{t} = s, \tilde{a}_{t} = a)
    \mathbb{P}^{\pi}_m(\tilde{s}_{t+1} = s'| \tilde{s}_{t} = s, \tilde{a}_{t} = a)}{\mathbb{P}^{\pi}_m(\tilde{a}_{t+1} = a' |  \tilde{s}_{t} = s, \tilde{a}_{t} = a)}\\
    &= \frac{\mathbb{P}^{\pi}_m(\tilde{a}_{t+1} = a' | \tilde{s}_{t} = s, \tilde{a}_{t} = a)
    \mathbb{P}^{\pi}_m(\tilde{s}_{t+1} = s'| \tilde{s}_{t} = s, \tilde{a}_{t} = a)}{\mathbb{P}^{\pi}_m(\tilde{a}_{t+1} = a' |  \tilde{s}_{t} = s, \tilde{a}_{t} = a)}\\
    &=\mathbb{P}^{\pi}_m(\tilde{s}_{t+1} = s' |  \tilde{s}_{t} = s, \tilde{a}_{t} = a).
\end{align*}
\end{proof}

\begin{theorem*}
Let $\pi := (d_0,d_1,\cdots)\in \Pi^{\textsc{HR}}$ be a history dependent policy. For all $s_0\in\St$, there exists a Markov policy $\pi' := (d_0',d_1',\cdots)\in \Pi^{\textsc{MR}}$  that yields the same process distribution as $\pi$, \ie for all $a\in\A, s'\in \St, t\geq m, $
\begin{equation}
\label{eq: HD_to_MD}
  \mathbb{P}^{\pi'}_m(\tilde{s}_{t-m}  = s', \tilde{a}_t  = a | \tilde{s}_0  = s_0) = 
\mathbb{P}^{\pi}_m(\tilde{s}_{t-m}  = s', \tilde{a}_t  = a | \tilde{s}_0  = s_0).
\end{equation}
\end{theorem*}
\begin{proof}
When $m= 0$, the result holds true by standard RL theory \citep{puterman2014markov}[Thm 5.5.1]. Thus, assume that $m >0$. 
Fix $s\in \St$. Let $\pi' :=(d_0',d_1',\cdots)$ with $d_0': \{s\} \rightarrow \Delta_{\A}$ defined  as 
\begin{equation}
    \label{eq: d_0 tag}
    q_{d_{0}'(s)}(a):= \mathbb{P}^{\pi}_m(\tilde{a}_m = a | \tilde{s}_0  = s )
\end{equation} 
and for all $t > m$,
\begin{equation}
    \label{eq: policy_tag_markov}
    \begin{split}
        q_{d_{t-m}'(s')}(a) := \mathbb{P}^{\pi}_m(\tilde{a}_t = a |\tilde{s}_{t-m}  =s', \tilde{s}_0  = s ), \quad \forall s'\in\St, a\in\A.
    \end{split}
\end{equation}
For the policy $\pi'$ defined as in Eqs.~\eqref{eq: d_0 tag}-\eqref{eq: policy_tag_markov}, we prove Eq.~\eqref{eq: HD_to_MD} by induction on $t\geq m$. By construction of $\pi'$, the induction base is satisfied at $t=m$. By construction of $\pi'$ again, for all $t > m$ we have
\begin{align}
     \mathbb{P}^{\pi'}_m(\tilde{a}_{t} = a |\tilde{s}_{t-m} = s',\tilde{s}_0  = s ) &= \mathbb{P}^{\pi'}_m(\tilde{a}_{t} = a |\tilde{s}_{t-m}  =s') \nonumber \\
     &= q_{d_{t-m}'(s')}(a) \nonumber\\
    &= \mathbb{P}^{\pi}_m(\tilde{a}_{t} = a |\tilde{s}_{t-m} =s',  \tilde{s}_0  = s )  \label{eq:action given s_0}.
\end{align}
Assume that Eq.~\eqref{eq: HD_to_MD} holds up until $t = n-1$. Further let the Euclidean division $n-1 = km + r$ of $n-1$ by $m$, so that $k,r\in\mathbb{N}$ with $0\leq r< m$. Then, we can write
\scriptsize{
\begin{align*}
&\mathbb{P}^{\pi}_m(\tilde{s}_{n}  = s' | \tilde{s}_0  = s)\\
&= \sum_{\substack{s_{km+r}\in\St,\\ a_{km+r}\in\A}}\mathbb{P}^{\pi}_m(\tilde{s}_{n}  = s' , \tilde{s}_{km+r}  = s_{km+r}, \tilde{a}_{km+r}  = a_{km+r}|\tilde{s}_0  = s) \\
&= \sum_{\substack{s_{km+r}\in\St,\\ a_{km+r}\in\A}}\mathbb{P}^{\pi}_m(\tilde{s}_{n}  = s' |\tilde{a}_{km+r}  = a_{km+r}, \tilde{s}_{km+r}  = s_{km+r},   \tilde{s}_0  = s)\\
&\qquad\qquad \mathbb{P}^{\pi}_m(\tilde{a}_{km+r}  = a_{km+r}| \tilde{s}_{km+r}  = s_{km+r},   \tilde{s}_0  = s) \mathbb{P}^{\pi}_m(\tilde{s}_{km + r}  = s_{km+r} | \tilde{s}_0  = s)\\
&= \sum_{\substack{s_{km+r}\in\St,\\ a_{km+r}\in\A}}P(s' | s_{km+r}, a_{km+r})
 \mathbb{P}^{\pi}_m(\tilde{a}_{km+r}  = a_{km+r}| \tilde{s}_{km+r}  = s_{km+r},   \tilde{s}_0  = s)\mathbb{P}^{\pi}_m(\tilde{s}_{km + r}  = s_{km+r} | \tilde{s}_0  = s).
\end{align*}}
\normalsize By Eq.~\eqref{eq: delayed policy}, $\tilde{a}_{km+r}$ only depends on history up to $(k-1)m+r$. Thus,
$\mathbb{P}^{\pi}_m(\tilde{a}_{km+r} = a_{km+r}| \tilde{s}_{km+r}  = s_{km+r}, \tilde{s}_{0}  = s )=\mathbb{P}^{\pi}_m(\tilde{a}_{km+r} = a_{km+r}| \tilde{s}_{0} = s)$ and
\begin{align*}
&\mathbb{P}^{\pi}_m(\tilde{s}_{n}  = s' | \tilde{s}_0  = s)\\ &=\sum_{\substack{s_{km+r}\in\St,\\ a_{km+r}\in\A}}P(s' | s_{km+r}, a_{km+r}) \mathbb{P}^{\pi}_m(\tilde{a}_{km+r}  = a_{km+r}| \tilde{s}_0  = s) \mathbb{P}^{\pi}_m(\tilde{s}_{km + r}  = s_{km+r} | \tilde{s}_0  = s).
\end{align*}
Since $km+r=n-1$, by the induction hypothesis we can rewrite
\begin{align*}
\mathbb{P}^{\pi}_m(\tilde{a}_{km+r}  = a_{km+r}| \tilde{s}_0  = s)  
&= \sum_{s_{(k-1)m+r}\in\St}\mathbb{P}^{\pi}_m(\tilde{a}_{km+r}  = a_{km+r}, \tilde{s}_{(k-1)m + r}  = s_{(k-1)m+r}| \tilde{s}_0  = s)\\
&= \sum_{s_{(k-1)m+r}\in\St}\mathbb{P}^{\pi'}_m(\tilde{a}_{km+r}  = a_{km+r}, \tilde{s}_{(k-1)m + r}  = s_{(k-1)m+r}| \tilde{s}_0  = s)\\
&= \mathbb{P}^{\pi'}_m(\tilde{a}_{km+r}  = a_{km+r}| \tilde{s}_0  = s),  
\end{align*}
so that 
\begin{align*}
&\mathbb{P}^{\pi}_m(\tilde{s}_{n}  = s' | \tilde{s}_0  = s)\\ &=\sum_{\substack{s_{km+r}\in\St,\\ a_{km+r}\in\A}}P(s' | s_{km+r}, a_{km+r}) \mathbb{P}^{\pi'}_m(\tilde{a}_{km+r}  = a_{km+r}| \tilde{s}_0  = s) \mathbb{P}^{\pi}_m(\tilde{s}_{km + r}  = s_{km+r} | \tilde{s}_0  = s).
\end{align*}
We now study the last term in the above equation, $\mathbb{P}^{\pi}_m(\tilde{s}_{km + r}  = s_{km+r} | \tilde{s}_0).$ We have
\scriptsize{\scriptsize{\begin{align*}
 &\mathbb{P}^{\pi}_m(\tilde{s}_{km + r}  = s_{km+r} | \tilde{s}_0  = s)\nonumber\\
&= \sum_{\substack{s_{km + r-1},\cdots, s_{km}\in\St\\ a_{km + r-1},\cdots, a_{km}\in\A}}\mathbb{P}^{\pi}_m(\tilde{s}_{km + r}  = s_{km+r},\tilde{s}_{km + r-1} = s_{km + r-1}, \tilde{a}_{km + r-1}= a_{km + r-1}, \cdots, \tilde{s}_{km} = s_{km}, \tilde{a}_{km} = a_{km}| \tilde{s}_{0}  = s )\nonumber\\
&= \sum_{\substack{s_{km + r-1},\cdots, s_{km}\in\St\\ a_{km + r-1},\cdots, a_{km}\in\A}}\mathbb{P}^{\pi}_m(\tilde{s}_{km + r}  = s_{km+r}|\tilde{s}_{km + r-1} = s_{km + r-1}, \tilde{a}_{km + r-1}= a_{km + r-1},  \cdots, \tilde{s}_{km} = s_{km}, \tilde{a}_{km} = a_{km}, \tilde{s}_{0}  = s )\nonumber\\
&\quad \mathbb{P}^{\pi}_m(\tilde{s}_{km + r-1} = s_{km + r-1}, \tilde{a}_{km + r-1}= a_{km + r-1},  \cdots, \tilde{s}_{km} = s_{km}, \tilde{a}_{km} = a_{km}| \tilde{s}_{0}  = s)\nonumber\\
&= \sum_{\substack{s_{km + r-1},\cdots, s_{km}\in\St\\ a_{km + r-1},\cdots, a_{km}\in\A}}P( s_{km+r}| s_{km + r-1},  a_{km + r-1})\nonumber\\
&\quad\mathbb{P}^{\pi}_m(\tilde{s}_{km + r-1} = s_{km + r-1}, \tilde{a}_{km + r-1}= a_{km + r-1},  \cdots, \tilde{s}_{km} = s_{km}, \tilde{a}_{km} = a_{km}| \tilde{s}_{0}  = s)\nonumber\\
&= \sum_{\substack{s_{km + r-1},\cdots, s_{km}\in\St\\ a_{km + r-1},\cdots, a_{km}\in\A}}P( s_{km+r}| s_{km + r-1},  a_{km + r-1})\nonumber\\
&\quad \mathbb{P}^{\pi}_m(\tilde{s}_{km + r-1} = s_{km + r-1}| \tilde{a}_{km + r-1}= a_{km + r-1}, \tilde{s}_{km + r-2}= s_{km + r-2}, \tilde{a}_{km + r-2}= a_{km + r-2},  \cdots, \tilde{s}_{km} = s_{km}, \tilde{a}_{km} = a_{km}, \tilde{s}_{0}  = s)\nonumber\\
&\quad\mathbb{P}^{\pi}_m(\tilde{a}_{km + r-1}= a_{km + r-1}, \tilde{s}_{km + r-2}= s_{km + r-2}, \tilde{a}_{km + r-2}= a_{km + r-2},  \cdots, \tilde{s}_{km} = s_{km}, \tilde{a}_{km} = a_{km}|\tilde{s}_{0}  = s)\nonumber\\
&\overset{\textrm{Lemma~\ref{lemma: independency}}}{=} \sum_{\substack{s_{km + r-1},\cdots, s_{km}\in\St\\ a_{km + r-1},\cdots, a_{km}\in\A}}P( s_{km+r}| s_{km + r-1},  a_{km + r-1})\nonumber\\
&\quad \mathbb{P}^{\pi}_m(\tilde{s}_{km + r-1} = s_{km + r-1}| \tilde{s}_{km + r-2}= s_{km + r-2}, \tilde{a}_{km + r-2}= a_{km + r-2}, \cdots, \tilde{s}_{km} = s_{km}, \tilde{a}_{km} = a_{km}, \tilde{s}_{0}  = s)\nonumber\\
&\quad \mathbb{P}^{\pi}_m(\tilde{a}_{km + r-1}= a_{km + r-1}, \tilde{s}_{km + r-2}= s_{km + r-2}, \tilde{a}_{km + r-2}= a_{km + r-2}, \cdots, \tilde{s}_{km} = s_{km}, \tilde{a}_{km} = a_{km}|\tilde{s}_{0}  = s)\nonumber\\
&= \sum_{\substack{s_{km + r-1},\cdots, s_{km}\in\St\\ a_{km + r-1},\cdots, a_{km}\in\A}}P( s_{km+r}| s_{km + r-1},  a_{km + r-1})P(s_{km + r-1}|  s_{km + r-2}, a_{km + r-2} ) \nonumber\\
&\quad  \mathbb{P}^{\pi}_m(\tilde{a}_{km + r-1}= a_{km + r-1}, \tilde{s}_{km + r-2}= s_{km + r-2}, \tilde{a}_{km + r-2}= a_{km + r-2}, \cdots, \tilde{s}_{km} = s_{km}, \tilde{a}_{km} = a_{km}|\tilde{s}_{0}  = s)\nonumber\\
&= \sum_{\substack{s_{km + r-1},\cdots, s_{km}\in\St\\ a_{km + r-1},\cdots, a_{km}\in\A}}P( s_{km+r}| s_{km + r-1},  a_{km + r-1})P(s_{km + r-1}|  s_{km + r-2}, a_{km + r-2} ) \nonumber\\
&\quad  \mathbb{P}^{\pi}_m(\tilde{a}_{km + r-1}= a_{km + r-1}| \tilde{s}_{km + r-2}= s_{km + r-2}, \tilde{a}_{km + r-2}= a_{km + r-2} \cdots, \tilde{s}_{km} = s_{km}, \tilde{a}_{km} = a_{km}|\tilde{s}_{0}  = s)\nonumber\\
&\quad \mathbb{P}^{\pi}_m(\tilde{s}_{km + r-2}= s_{km + r-2}, \tilde{a}_{km + r-2}= a_{km + r-2} \cdots, \tilde{s}_{km} = s_{km}, \tilde{a}_{km} = a_{km}|\tilde{s}_{0}  = s)\nonumber\\
&=\nonumber\\
&\vdots\nonumber\\
&= \sum_{\substack{s_{km + r-1},\cdots, s_{km}\in\St\\ a_{km + r-1},\cdots, a_{km}\in\A}}\biggr( \prod_{i = 1}^{r} P(s_{km+i}|s_{km+i-1}, a_{km+i-1}) \biggr) \nonumber\\
&\quad\biggr( \prod_{j = 1}^{r-1}\mathbb{P}^{\pi}_m(\tilde{a}_{km+j} = a_{km+j}|\tilde{s}_{km + j-1} = s_{km + j-1},\tilde{a}_{km + j-1}= a_{km + j-1},\cdots, \tilde{s}_{km} = s_{km}, \tilde{a}_{km} = a_{km} ,\tilde{s}_{0}  = s )\biggr)\nonumber\\
&\quad\mathbb{P}^{\pi}_m(\tilde{s}_{km}  = s_{km}|\tilde{a}_{km} = a_{km}, \tilde{s}_{0}  = s)\mathbb{P}^{\pi}_m(\tilde{a}_{km} = a_{km}| \tilde{s}_{0}  = s)\nonumber\\
&\overset{(1)}{=}\sum_{\substack{s_{km + r-1},\cdots, s_{km}\in\St\\ a_{km + r-1},\cdots, a_{km}\in\A}}\biggr( \prod_{i = 1}^{r} P(s_{km+i}|s_{km+i-1}, a_{km+i-1}) \biggr)\biggr( \prod_{j = 1}^{r-1}\mathbb{P}^{\pi}_m(\tilde{a}_{km+j} = a_{km+j}|\tilde{s}_{0}  = s )\biggr)\nonumber\\
&\quad\mathbb{P}^{\pi}_m(\tilde{s}_{km}  = s_{km}|\tilde{a}_{km} = a_{km}, \tilde{s}_{0}  = s)\mathbb{P}^{\pi}_m(\tilde{a}_{km} = a_{km}| \tilde{s}_{0}  = s)\nonumber\\
&=\sum_{\substack{s_{km + r-1},\cdots, s_{km}\in\St\\ a_{km + r-1},\cdots, a_{km}\in\A}}\biggr( \prod_{i = 1}^{r} P(s_{km+i}|s_{km+i-1}, a_{km+i-1})  \mathbb{P}^{\pi}_m(\tilde{a}_{km+i-1} = a_{km+i-1}|\tilde{s}_{0}  = s ) \biggr) \nonumber\\
&\quad\mathbb{P}^{\pi}_m(\tilde{s}_{km}  = s_{km}|\tilde{a}_{km} = a_{km}, \tilde{s}_{0}  = s)\nonumber\\
&=\sum_{\substack{s_{km + r-1},\cdots, s_{km}\in\St\\ a_{km + r-1},\cdots, a_{km}\in\A}} \biggr(\prod_{i = 1}^{r} P(s_{km+i}|s_{km+i-1}, a_{km+i-1})\\
&\quad \biggr(\sum_{s_{(k-1)m+i-1}'\in\St}\mathbb{P}^{\pi}_m(\tilde{a}_{km+i-1} = a_{km+i-1}, \tilde{s}_{(k-1)m+i-1} = s_{(k-1)m+i-1}'|\tilde{s}_{0}  = s )\biggr)\biggr)\nonumber\\
&\quad \mathbb{P}^{\pi}_m(\tilde{s}_{km}  = s_{km}|\tilde{a}_{km} = a_{km}, \tilde{s}_{0}  = s)\nonumber
\end{align*}
\begin{align*}
&\overset{(2)}{=}\sum_{\substack{s_{km + r-1},\cdots, s_{km}\in\St\\ a_{km + r-1},\cdots, a_{km}\in\A}} \biggr(\prod_{i = 1}^{r} P(s_{km+i}|s_{km+i-1}, a_{km+i-1})\\
&\quad \biggr(\sum_{s_{(k-1)m+i-1}'\in\St}\mathbb{P}^{\pi'}_m(\tilde{a}_{km+i-1} = a_{km+i-1}, \tilde{s}_{(k-1)m+i-1} = s_{(k-1)m+i-1}'|\tilde{s}_{0}  = s )\biggr)\biggr)\nonumber\\
&\quad \mathbb{P}^{\pi}_m(\tilde{s}_{km}  = s_{km}|\tilde{a}_{km} = a_{km}, \tilde{s}_{0}  = s)\nonumber\\
&=\sum_{\substack{s_{km + r-1},\cdots, s_{km}\in\St\\ a_{km + r-1},\cdots, a_{km}\in\A}}\biggr( \prod_{i = 1}^{r} P(s_{km+i}|s_{km+i-1}, a_{km+i-1})  \mathbb{P}^{\pi'}_m(\tilde{a}_{km+i-1} = a_{km+i-1}|\tilde{s}_{0}  = s ) \biggr) \nonumber\\
&\quad\mathbb{P}^{\pi}_m(\tilde{s}_{km}  = s_{km}|\tilde{a}_{km} = a_{km}, \tilde{s}_{0}  = s)\nonumber\\
&\overset{(3)}{=} \sum_{\substack{s_{km + r-1},\cdots, s_{km}\in\St\\ a_{km + r-1},\cdots, a_{km}\in\A}}\biggr( \prod_{i = 1}^{r} P(s_{km+i}|s_{km+i-1}, a_{km+i-1})  \mathbb{P}^{\pi'}_m(\tilde{a}_{km+i-1} = a_{km+i-1}|\tilde{s}_{0}  = s ) \biggr) \nonumber\\
&\quad\mathbb{P}^{\pi}_m(\tilde{s}_{km}  = s_{km}| \tilde{s}_{0}  = s).\nonumber
\end{align*}}}
\normalsize
In $(1)$, we use Eq.~\eqref{eq: delayed policy} to establish that $\tilde{a}_{km+j}$ only depends on history up to $(k-1)m+j$. Since $m-1>r-1 \geq j \geq 1$, we have $km >(k-1)m+j$, and
\begin{align*}
&\mathbb{P}^{\pi}_m(\tilde{a}_{km+j} = a_{km+j}|\tilde{s}_{km+j-1} = s_{km+j-1}, \tilde{a}_{km+j-1} = a_{km+j-1},\cdots ,\tilde{s}_{km} = s_{km}, \tilde{a}_{km} = a_{km}, \tilde{s}_{0}  = s )\\
&=\mathbb{P}^{\pi}_m(\tilde{a}_{km+j} = a_{km+j}| \tilde{s}_{0}  = s ).
\end{align*}
In $(2)$, we use the induction hypothesis.
In $(3)$ we use Bayes rule and Eq.~\eqref{eq: delayed policy} again to obtain:
\begin{align*}
\mathbb{P}^{\pi}_m(\tilde{s}_{km}  = s_{km}|\tilde{a}_{km} = a_{km}, \tilde{s}_{0}  = s)
  &= \frac{\mathbb{P}^{\pi}_m(\tilde{a}_{km} = a_{km}|\tilde{s}_{km}  = s_{km}, \tilde{s}_{0}  = s)\mathbb{P}^{\pi}_m(\tilde{s}_{km}  = s_{km}|\tilde{s}_{0}  = s)}{\mathbb{P}^{\pi}_m(\tilde{a}_{km} = a_{km}|\tilde{s}_{0}  = s)}\\
  &= \frac{\mathbb{P}^{\pi}_m(\tilde{a}_{km} = a_{km}| \tilde{s}_{0}  = s)\mathbb{P}^{\pi}_m(\tilde{s}_{km}  = s_{km}|\tilde{s}_{0}  = s)}{\mathbb{P}^{\pi}_m(\tilde{a}_{km} = a_{km}|\tilde{s}_{0}  = s)}\\
  &=\mathbb{P}^{\pi}_m(\tilde{s}_{km}  = s_{km}|\tilde{s}_{0}  = s).
\end{align*}
Thus, it results that 
\begin{align*}
&\mathbb{P}^{\pi}_m(\tilde{s}_{n}  = s' | \tilde{s}_0  = s)\\ 
&=\sum_{\substack{s_{km+r}\in\St,\\ a_{km+r}\in\A}}P(s' | s_{km+r}, a_{km+r}) \mathbb{P}^{\pi'}_m(\tilde{a}_{km+r}  = a_{km+r}| \tilde{s}_0  = s)\\
&\sum_{\substack{s_{km + r-1},\cdots, s_{km}\in\St\\ a_{km + r-1},\cdots, a_{km}\in\A}}\biggr( \prod_{i = 1}^{r} P(s_{km+i}|s_{km+i-1}, a_{km+i-1})  \mathbb{P}^{\pi'}_m(\tilde{a}_{km+i-1} = a_{km+i-1}|\tilde{s}_{0}  = s ) \biggr) \nonumber\\
&\quad\mathbb{P}^{\pi}_m(\tilde{s}_{km}  = s_{km}| \tilde{s}_{0}  = s)\\
&=\sum_{\substack{s_{km + r},\cdots, s_{km}\in\St\\ a_{km + r},\cdots, a_{km}\in\A}}\biggr( \prod_{i = 1}^{r+1} P(s_{km+i}|s_{km+i-1}, a_{km+i-1})  \mathbb{P}^{\pi'}_m(\tilde{a}_{km+i-1} = a_{km+i-1}|\tilde{s}_{0}  = s ) \biggr)\\
&\quad\mathbb{P}^{\pi}_m(\tilde{s}_{km}  = s_{km}| \tilde{s}_{0}  = s),
\end{align*}
where we used the convention $s_{km+r+1} = s_{n} = s'$.
We similarly use backward induction until the remaining term that depends on $\pi$ becomes 
\begin{align*}
    &\mathbb{P}^{\pi}_m(\tilde{s}_{m}  = s_{m}| \tilde{s}_{0}  = s)\\
    &= \sum_{\substack{s_{m-1},\cdots,s_1\in\St\\ \substack{a_{m-1},\cdots, a_0\in\A}}} \mathbb{P}^{\pi}_m(\tilde{s}_{m}  = s_{m}, \tilde{s}_{m-1}  = s_{m-1}, \tilde{a}_{m-1}  = a_{m-1},\cdots, \tilde{s}_{1}  = s_{1}, \tilde{a}_{1}  = a_{1}, \tilde{a}_{0}  = a_{0}| \tilde{s}_{0}  = s)\\
    &= \sum_{\substack{s_{m-1},\cdots,s_1\in\St\\ \substack{a_{m-1},\cdots, a_0\in\A}}} \frac{1}{\mathbb{P}^{\pi}_m( \tilde{s}_{0}  = s)}\mathbb{P}^{\pi}_m(\tilde{s}_{m}  = s_{m}, \tilde{s}_{m-1}  = s_{m-1}, \tilde{a}_{m-1}  = a_{m-1},\cdots, \tilde{s}_{1}  = s_{1}, \tilde{a}_{1}  = a_{1}, \tilde{a}_{0}  = a_{0}, \tilde{s}_{0}  = s)\\
    &\overset{(4)}{=} \sum_{\substack{s_{m-1},\cdots,s_1\in\St\\ \substack{a_{m-1},\cdots, a_0\in\A}}}\frac{1}{\mu(s)}\mu(s)\left(\prod_{i = 0}^{m-1} P(s_{i+1}|s_{i}, a_{i})\delta_{\bar{a}_i}( a_{i})\right)
    =\sum_{\substack{s_{m-1},\cdots,s_1\in\St\\ \substack{a_{m-1},\cdots, a_0\in\A}}}\left(\prod_{i = 0}^{m-1} P(s_{i+1}|s_{i}, a_{i})\delta_{\bar{a}_i}( a_{i})\right),
\end{align*}
where $(4)$ results from Prop.~\ref{proposition: delayed process distribution}. Since the obtained quantity is independent of $\pi$, we have
$$
\mathbb{P}^{\pi}_m(\tilde{s}_{m}  = s_{m}|\tilde{s}_{0}  = s) = \mathbb{P}^{\pi'}_m(\tilde{s}_{m}  = s_{m}| \tilde{s}_{0}  = s).
$$
Thus, if we decompose $\mathbb{P}^{\pi'}_m(\tilde{s}_{n}  = s' | \tilde{s}_0  = s)$ according to the exact same derivation as we did for  $\mathbb{P}^{\pi}_m(\tilde{s}_{n}  = s' | \tilde{s}_0  = s)$, we obtain that at $t=n$,
\begin{equation}
\label{eq: s_t given s_0}
    \mathbb{P}^{\pi}_m(\tilde{s}_{n}  = s' | \tilde{s}_0  = s) = \mathbb{P}^{\pi'}_m(\tilde{s}_{n}  = s' | \tilde{s}_0  = s).
\end{equation}
As a result, at $t=n$ we have
\begin{align*}
\mathbb{P}^{\pi'}_m(\tilde{s}_{n-m}  = s', \tilde{a}_{n}  = a | \tilde{s}_0  = s) 
&{=} \mathbb{P}^{\pi'}_m( \tilde{a}_{n}  = a | \tilde{s}_{n-m}  = s', \tilde{s}_0  = s)\mathbb{P}^{\pi'}_m( \tilde{s}_{n-m}  = s'| \tilde{s}_0  = s)\\
&\overset{(a)}{=} \mathbb{P}^{\pi'}_m( \tilde{a}_{n}  = a | \tilde{s}_{n-m}  = s', \tilde{s}_0  = s) \mathbb{P}^{\pi}_m( \tilde{s}_{n-m}  = s'| \tilde{s}_0  = s)\\
&\overset{(b)}{=} \mathbb{P}^{\pi}_m(\tilde{a}_{n} = a |\tilde{s}_{n-m}  =s', \tilde{s}_0  = s )\mathbb{P}^{\pi}_m( \tilde{s}_{n-m}  = s'| \tilde{s}_0  = s)\\
&\overset{(c)}{=} \mathbb{P}^{\pi}_m(\tilde{s}_{n-m}  = s', \tilde{a}_{n}  = a | \tilde{s}_0  = s),
\end{align*}
where $(b)$ follows from Eq.~\eqref{eq: s_t given s_0};
$(c)$ from Eq.~\eqref{eq:action given s_0}.
Finally, assuming it is satisfied at $t=n-1$, the induction step is proved for $t=n$, which ends the proof. 
\end{proof}

\subsection{Degradation due to Stationarity}
\label{sec: degradation stationarity}
Prop.~\ref{prop: degradation due to stationarity} follows from computing the optimal return on an execution-delay MDP (EDMDP) using simulation. Specifically, we use Example~\ref{exmp: 2-state MDP} which we analytically studied in Sec.~\ref{sec: ED MDP}. We exhaustively search over the deterministic policy spaces $\Pi^{\textsc{SD}}$ and  $\Pi^{\textsc{MD}}$ to find the optimum.
We stress that limiting our search to deterministic policies is sufficient for this MDP. Indeed, as shown in Appx.~\ref{sec: example proof}, optimal return is attained for a deterministic policy when maximizing over $\Pi^{\textsc{SR}}$. Regarding the non-stationary Markov policy space $\Pi^{\textsc{MR}}$, as we have proved in Thm.~\ref{thm: optimal deterministic policy}, there exists an optimal deterministic policy in $\Pi^{\textsc{MD}}$. We show with this experiment that this optimal Markov deterministic policy attains better return than any stationary policy.
We set $p=0.8.$ Since the search-space of non-stationary policies is exponential in the simulation horizon ($T=10$ here), we choose $\gamma=0.5$ to have low approximation error. 

\begin{table}[h]
    \centering
    {\scriptsize
    \begin{tabular}{||c c c c c c c||} 
 \hline
  Policy-Type & $m=0$ & $m=1$ & $m=2$ & $m=3$ & $m=4$ & $m=5$ \\ [0.5ex] 
 \hline\hline
 Stationary (theoretical) & $2$ & $1.6$ & $1.36$ & $1.216$ & $1.129$ & $1.077$\\ 
 \hline
 Stationary & $1.99 \pm 0.01$ & $1.59 \pm 0.03$ & $1.32 \pm 0.05$ & $1.22 \pm 0.08$ & $1.11 \pm 0.09$ & $1.02 \pm 0.13$ \\
 \hline
 Non-stationary Markov & $1.99 \pm 0.01$ & $1.82 \pm 0.05$ & $1.67 \pm 0.08$ & $1.59 \pm 0.12$ & $1.46 \pm 0.15$ & $1.38 \pm 0.2$\\
 [1ex] 
 \hline
\end{tabular}
}
 \caption{Optimal return for different delay values and policy types.}
  \label{table: stationarity deg}
\end{table}

The results are summarized in Table~\ref{table: stationarity deg}. Apart from demonstrating sub-optimality of the stationary policy, they also confirm that our theoretical return for the stationary policy $\frac{1 + (2p - 1)^m}{2(1-\gamma)}$ from Prop.~\ref{prop: degredation} matches closely with simulation.  

\subsection{The Delayed Value Function}
\label{appx: delayed value function}
Given a random variable $W$ over $(\Omega, \mathcal{B}(\Omega), \mathbb{P}_m^{\pi})$, its expectation is
$
\label{expectation_W}
 \mathbb{E}_{m}^{\pi}[W] = \sum_{\omega\in\Omega}W(\omega)\mathbb{P}^{\pi}_m(\omega),
$
where $\omega=(s_0, a_0, s_1,\cdots)$ is a sample path. A typical $W$ to consider is the discounted sum of rewards
$W(s_0, a_0, s_1,\cdots):= \sum_{t=0}^{\infty} \gamma^t r(s_t, a_t)$. Thus, the expectation conditioned on initial state $s_0$ is given by
$\mathbb{E}_{m}^{\pi}[W|s_0] = \sum_{\omega\in\Omega}W(s_0, a_0, \cdots)\mathbb{P}^{\pi}_m(s_{0}, a_{0}, s_{1},\cdots|s_0).$
Let the \emph{delayed value function} 
\begin{equation}
\label{def: delayed value function}
    v^{\mu_0:\mu_{m-1} , \pi}_{m}(s_0) := \mathbb{E}_{m}^{\pi}\left[\sum_{ t= m}^{\infty} \gamma^{t-m} r(\tilde{s}_t,\tilde{a}_t)\biggr | \tilde{s}_0 = s_0 \right],
\end{equation}
where $\mu_0:\mu_{m-1} := (\mu_0, \cdots, \mu_{m-1})$ denotes some fixed queue of action distributions according to which the initial $m$ actions should be executed.
Note that the definition of $W$ does not change w.r.t. the delay value $m$: it always denotes the discounted sum of rewards. However, its distribution does depend on the delay value $m$ through the process distribution $\mathbb{P}^{\pi}_m$ and, as a result, so does its expectation $\mathbb{E}_{m}^{\pi}[W|s_0]$. 

Consider a Markov policy $\pi:=(d_k)_{k\geq 0}\in\Pi^{\textsc{MR}}$.  
For all $s,s'\in\St, k\in\mathbb{N}$ and $u\in\Delta_{\A}$, let $P_{u}(s, s') := \sum_{a\in\A}u(a)P(s'| s,a)$ and $R_{d_k}(s',s) := \sum_{a\in\A}q_{d_k(s)}(a)r(s',a)$. 
We then have the following result. 

\begin{theorem}
\label{thm: delayed value function}
For a Markov policy $\pi\in\Pi^{\textsc{MR}}$ given by $\pi:=(d_0,d_1,\cdots)$, the delayed value function satisfies the following relation:

\begin{align*}
    v^{\mu_0:\mu_{m-1},\pi}_{m}(s_0) &=    
    \left(P_{\mu_0} \cdots P_{\mu_{m-1}}R_{d_0}\right) (s_0,s_0)
    + \gamma    \sum_{\substack{ s_1,  \cdots, s_{m}\in\St } } \left( \prod_{k=0}^{m-1} P_{\mu_k} (s_k, s_{k+1})\right) v^{d_0(s_0): d_{m-1}(s_{m-1}),\pi_{m:}}_{m}(s_m),
\end{align*}
where $\pi_{m:}:=(d_m, d_{m+1},\cdots)$ denotes the policy $\pi$ starting from its $m+1$-th decision rule.
\end{theorem}

In addition, this relation becomes a recursion when the policy is  $m$-periodic. Its proof is omitted since the result immediately follows from Thm.~\ref{thm: delayed value function}.  

\begin{corollary}
For an $m$-periodic Markov policy $\pi\in\Pi^{\textsc{MR}}$ given by ${\pi:=(d_0,\cdots,d_{m-1},d_0,\cdots,d_{m-1},\cdots)}$, the delayed value function satisfies the following recursion:
\begin{align*}
    v^{\mu_0:\mu_{m-1},\pi}_{m}(s_0) &=    
    \left(P_{\mu_0} \cdots P_{\mu_{m-1}}R_{d_0}\right) (s_0,s_0)
    + \gamma    \sum_{\substack{ s_1,  \cdots, s_{m}\in\St } } \left( \prod_{k=0}^{m-1} P_{\mu_k} (s_k, s_{k+1})\right) v^{d_0(s_0): d_{m-1}(s_{m-1}),\pi}_{m}(s_m).
\end{align*}
\end{corollary}

\begin{proof}[Proof of Theorem~\ref{thm: delayed value function}]
By definition of the delayed value function we have:
\begin{align*}
  v^{\mu_0:\mu_{m-1},\pi}_{m}(s_0) &= \mathbb{E}_{m}^{\pi}\left[\sum_{ t= m}^{\infty} \gamma^{t-m} r(\tilde{s}_t,\tilde{a}_t)\biggr | \tilde{s}_0 = s_0 \right]\\
  &\overset{(1)}{=} \sum_{ t= m}^{\infty} \gamma^{t-m}  \mathbb{E}_{m}^{\pi}\left[ r(\tilde{s}_t,\tilde{a}_t)\biggr | \tilde{s}_0 = s_0 \right]\\
  &\overset{(2)}{=} \sum_{ t= m}^{\infty} \gamma^{t-m} \sum_{\substack{s_1,\cdots, s_t\in\St \\a_0,\cdots,a_t\in\A}} r(s_t, a_t)\mathbb{P}_{m}^{\pi}(\tilde{a}_0 = a_0,\tilde{s}_1  = s_1, \cdots, \tilde{a}_{t-1}  = a_{t-1}, \tilde{s}_t  = s_t ,\tilde{a}_t = a_t| \tilde{s}_0 = s_0),
\end{align*}
where $(1)$ results from the dominated convergence theorem and $(2)$ by the definition of expectation. 
Using Prop.~\ref{proposition: delayed process distribution} and the fact that $\pi$ is a Markov policy, we can write the probability of a sample path conditioned on the initial state as:
    \begin{equation*}
    \begin{split}
     &\mathbb{P}_{m}^{\pi}(\tilde{a}_0 = a_0,\tilde{s}_1  = s_1, \cdots, \tilde{a}_{t-1}  = a_{t-1}, \tilde{s}_t  = s_t ,\tilde{a}_t = a_t| \tilde{s}_0 = s)\\
     &= \left(\prod_{k =0}^{m-1}\mu_k(a_k) P(s_{k+1}|s_k, a_k)\right)\left(\prod_{k=m}^{t-1} q_{d_{k-m}(s_{k-m})}(a_k)P(s_{k+1}|s_k, a_k)
  \right)q_{d_{t-m}(s_{t-m})}(a_t),
     \end{split}
    \end{equation*}
so that:
\begin{align*}
  v^{\mu_0:\mu_{m-1},\pi}_{m}(s_0)
  &=  \sum_{ t= m}^{\infty} \gamma^{t-m} \sum_{\substack{s_1,\cdots, s_t\in\St \\a_0,\cdots,a_t\in\A}} r(s_t, a_t)\left(\prod_{k =0}^{m-1}\mu_k(a_k) P(s_{k+1}|s_k, a_k)\right)\\
  &\quad \cdot \left(\prod_{k=m}^{t-1} q_{d_{k-m}(s_{k-m})}(a_k)P(s_{k+1}|s_k, a_k)
  \right)q_{d_{t-m}(s_{t-m})}(a_t).
\end{align*}

Then, we can rewrite the delayed value function as:
\begin{align*}
&v^{\mu_0:\mu_{m-1},\pi}_{m}(s_0)\\
&= \sum_{ t= m}^{\infty} \gamma^{t-m} \sum_{\substack{ s_1,  \cdots, s_t \in\St\\ a_0,\cdots, a_{t-1}\in\A} }R_{d_{t-m}}(s_t, s_{t-m}) \cdot
\left(\prod_{k=0}^{m-1} \mu_k(a_k) P_{a_k}(s_k, s_{k+1})\right) \left(\prod_{k=m}^{t-1} q_{d_{k-m}(s_{k-m})}(a_k)P_{a_k}(s_k, s_{k+1})
  \right).\\
&= \sum_{ t= m}^{\infty} \gamma^{t-m} \sum_{\substack{ s_1,  \cdots, s_t \in\St\\ a_m,\cdots, a_{t-2}\in\A} }R_{d_{t-m}}(s_t, s_{t-m}) \cdot
\left(\prod_{k=0}^{m-1} \mu_k(a_k) P_{a_k}(s_k, s_{k+1})\right) \left(\prod_{k=m}^{t-2} q_{d_{k-m}(s_{k-m})}(a_k)P_{a_k}(s_k, s_{k+1})
  \right)\cdot\\
  &\qquad\sum_{a_{t-1}\in\A}q_{d_{t-1-m}(s_{t-1-m})}(a_{t-1})P_{a_{t-1}}(s_{t-1}, s_{t})\\
  &=\sum_{ t= m}^{\infty} \gamma^{t-m} \sum_{\substack{ s_1,  \cdots, s_t \in\St\\ a_m,\cdots, a_{t-2}\in\A} }R_{d_{t-m}}(s_t, s_{t-m}) \cdot
\left(\prod_{k=0}^{m-1} \mu_k(a_k) P_{a_k}(s_k, s_{k+1})\right)\left(\prod_{k=m}^{t-2} q_{d_{k-m}(s_{k-m})}(a_k)P_{a_k}(s_k, s_{k+1})
  \right)\cdot\\
  &\qquad P_{d_{t-1-m}(s_{t-1-m})}(s_{t-1}, s_{t})\\
  &\vdots\\
  &= \sum_{ t= m}^{\infty} \gamma^{t-m} \sum_{\substack{ s_1,  \cdots, s_t \in\St} }R_{d_{t-m}}(s_t, s_{t-m}) \cdot
\left(\prod_{k=0}^{m-1} P_{\mu_k}(s_k, s_{k+1})\right)\left(\prod_{k=m}^{t-1} P_{d_{k-m}(s_{k-m})}(s_k, s_{k+1})\right),
\end{align*}
and the following can be derived:
\begin{align}
\label{eq: explicit delayed value}
v^{\mu_0:\mu_{m-1},\pi}_{m}(s_0)&=\sum_{t= m}^{\infty} \gamma^{t-m} \sum_{\substack{ s_1,  \cdots, s_t } } R_{d_{t-m}}(s_t, s_{t-m}) \left(\prod_{k=0}^{m-1} P_{\mu_k}(s_k, s_{k+1})\right) \left(\prod_{k=m}^{t-1} P_{d_{k-m}(s_{k-m})}(s_k, s_{k+1})\right)\nonumber\\
&=\sum_{\substack{ s_1,  \cdots, s_m } } R_{d_{0}}(s_m, s_0) \left(\prod_{k=0}^{m-1} P_{\mu_k}(s_k, s_{k+1})\right)  + f(s_0)\nonumber\\
&=\sum_{s_m}\left(P_{\mu_0} \cdots P_{\mu_{m-1}} \right)(s_0, s_m) R_{d_{0}}( s_m, s_0) + f(s_0)\nonumber\\
&=  \left(P_{\mu_0} \cdots P_{\mu_{m-1}} R_{d_{0}}\right) (s_0,s_0)  + f(s_0)\nonumber\\
&= \left(P_{\mu_0} \cdots P_{\mu_{m-1}}R_{d_0}\right) (s_0,s_0) + f(s_0),
\end{align}
where
\begin{align}
\label{eq: f s_0}
    f(s_0) :&= \sum_{t= m+1}^{\infty} \gamma^{t-m} \sum_{\substack{ s_1,  \cdots, s_t } } R_{d_{t-m}}(s_t, s_{t-m}) \left(\prod_{k=0}^{m-1} P_{\mu_k}(s_k, s_{k+1})\right) \left(\prod_{k=m}^{t-1} P_{d_{k-m}(s_{k-m})}(s_k, s_{k+1})\right)\nonumber\\
    &= \sum_{t= m+1}^{\infty} \gamma^{t-m} \sum_{\substack{ s_1,  \cdots, s_m } } \sum_{\substack{ s_{m+1},  \cdots, s_t} } R_{d_{t-m}}(s_t, s_{t-m}) \left(\prod_{k=0}^{m-1} P_{\mu_k}(s_k, s_{k+1})\right) \nonumber\\
    &\quad \cdot \left(\prod_{k=m}^{t-1} P_{d_{k-m}(s_{k-m})}(s_k, s_{k+1})\right)\nonumber\\
    &= \sum_{t= m+1}^{\infty} \gamma^{t-m} \sum_{\substack{ s_1,  \cdots, s_m } } \left(\prod_{k=0}^{m-1} P_{\mu_k}(s_k, s_{k+1})\right) \sum_{\substack{ s_{m+1},  \cdots, s_t} } R_{d_{t-m}}(s_t, s_{t-m})  \left(\prod_{k=m}^{t-1} P_{d_{k-m}(s_{k-m})}(s_k, s_{k+1})\right).
\end{align}
In fact, the last part of the sum corresponds to the following expectation: 
\begin{align*}
    \sum_{s_{m+1}, \cdots,s_t}R_{d_{t-m}}(s_t, s_{t-m})\left(\prod_{k=m}^{t-1} P_{d_{k-m}(s_{k-m})}(s_k, s_{k+1})\right)=\mathbb{E}_m^{\pi}\left[r(\tilde{s}_t, \tilde{a}_t)| \tilde{s}_0 = s_0, \tilde{s}_1 = s_1,\cdots, \tilde{s}_m = s_m\right], 
\end{align*}
so
\begin{align*}
    f(s_0)  &=  \sum_{t= m+1}^{\infty} \gamma^{t-m} \sum_{\substack{ s_1,  \cdots, s_m } } \left(\prod_{k=0}^{m-1} P_{\mu_k}(s_k, s_{k+1})\right) \mathbb{E}_m^{\pi}\left[r(\tilde{s}_t, \tilde{a}_t)| \tilde{s}_0 = s_0, \tilde{s}_1 = s_1,\cdots, \tilde{s}_m = s_m\right]
\end{align*}
and 
\begin{align*}
    v^{\mu_0:\mu_{m-1},\pi}_{m}(s_0)
    &= \left(P_{\mu_0} \cdots P_{\mu_{m-1}} R_{d_{0}}\right) (s_0,s_0)\\
    &\quad+ \sum_{t= m+1}^{\infty} \gamma^{t-m} \sum_{\substack{ s_1,  \cdots, s_m } } \left(\prod_{k=0}^{m-1} P_{\mu_k}(s_k, s_{k+1})\right) \mathbb{E}_m^{\pi}\left[r(\tilde{s}_t, \tilde{a}_t)| \tilde{s}_0 = s_0, \tilde{s}_1 = s_1,\cdots, \tilde{s}_m = s_m\right]\\
    &= \left(P_{\mu_0} \cdots P_{\mu_{m-1}} R^{\pi_{0}}\right) (s_0,s_0)\\
    &\quad+ \gamma \sum_{\substack{ s_1,  \cdots, s_m } } \left(\prod_{k=0}^{m-1} P_{\mu_k}(s_k, s_{k+1})\right) \mathbb{E}_m^{\pi}\left[\sum_{t= m+1}^{\infty} \gamma^{t-m-1} r(\tilde{s}_t, \tilde{a}_t)| \tilde{s}_0 = s_0, \tilde{s}_1 = s_1,\cdots, \tilde{s}_m = s_m\right]
\end{align*}
Finally, fixing initial states $s_0,\cdots,s_{m-1}$ implies fixing a queue of $m$ action distributions $d_0(s),\cdots,d_{m-1}(s_{m-1})$. Therefore, denoting by $\pi_{m:}:=(d_m, d_{m+1},\cdots)$ the original policy $\pi$ starting from its $m+1$-th decision rule, we have 
$$
\mathbb{E}_m^{\pi}\left[\sum_{t= m+1}^{\infty} \gamma^{t-m-1} r(\tilde{s}_t, \tilde{a}_t)| \tilde{s}_0 = s_0, \tilde{s}_1 = s_1,\cdots, \tilde{s}_m = s_m\right] = v^{d_0(s_0): d_{m-1}(s_{m-1}),\pi_{m:}}_{m}(s_m),
$$
and 
\begin{align*}
&v^{\mu_0:\mu_{m-1},\pi}_{m}(s_0)  \\
&=\left(P_{\mu_0} \cdots P_{\mu_{m-1}} R^{\pi_{0}}\right) (s_0,s_0) + \gamma \sum_{\substack{ s_1,  \cdots, s_m } } \left(\prod_{k=0}^{m-1} P_{\mu_k}(s_k, s_{k+1})\right)v^{d_0(s_0): d_{m-1}(s_{m-1}),\pi_{m:}}_{m}(s_m),
\end{align*}
which concludes the proof.
\end{proof}

\subsection{Proof of Theorem \ref{thm: optimal deterministic policy}}
\begin{theorem*}
For any action distribution queue $\mu_0:\mu_{m-1}:=(\mu_0,\dots,\mu_{m-1})$ and $s_0\in \St,$
\begin{equation}
    \label{eq: optimality of deterministic markov}
\max_{\pi \in \Pi^{\textsc{MD}}} v^{\mu_0:\mu_{m-1},\pi}_{m}= \max_{\pi \in \Pi^{\textsc{MR}}}v^{\mu_0:\mu_{m-1},\pi}_{m}.
\end{equation}
\end{theorem*}

\begin{proof}
First, since $\Pi^{\textsc{MD}} \subset \Pi^{\textsc{MR}},$ the RHS of \eqref{eq: optimality of deterministic markov} must be at least as great as the LHS. We now establish the reverse inequality.
Recall from Eqs.\eqref{eq: explicit delayed value}-\eqref{eq: f s_0}[Proof of Thm.~\ref{thm: delayed value function}] that 
\begin{align}
\label{eq: rewriting the value}
&v^{\mu_0:\mu_{m-1},\pi}_{m}(s_0) = \left(P_{\mu_0} \cdots P_{\mu_{m-1}} R_{d_{0}}\right) (s_0,s_0) \nonumber\\
  &+ \sum_{ t= m+1}^{\infty} \gamma^{t-m} \sum_{\substack{ s_1,  \cdots, s_t \in\St} }R_{d_{t-m}}(s_t, s_{t-m}) \cdot
\left(\prod_{k=0}^{m-1} P_{\mu_k}(s_k, s_{k+1})\right)\left(\prod_{k=m}^{t-1} P_{d_{k-m}(s_{k-m})}(s_k, s_{k+1})\right).
\end{align}
We shall now rewrite the value in two forms, to respectively show its dependence on $d_0(s_0),$ and $d_i(s_i)$ for $i\geq 1.$ For each of these two forms, we will prove a deterministic decision is at least as good as a random one in terms of value.

We begin with rewriting the first term in \eqref{eq: rewriting the value} as
\begin{align*}
    \left(P_{\mu_0} \cdots P_{\mu_{m-1}} R_{d_{0}}\right) (s_0,s_0) &=
    (\left(P_{\mu_0} \cdots P_{\mu_{m-1}}) R_{d_{0}}\right) (s_0,s_0)\\
    &= \sum_{s_m\in\St} (P_{\mu_0} \cdots P_{\mu_{m-1}})(s_0, s_m)R_{d_{0}}(s_m,s_0)\\
    &= \sum_{s_m\in\St} (P_{\mu_0} \cdots P_{\mu_{m-1}})(s_0, s_m)\left(\sum_{a_0\in\A}q_{d_0(s_0)}(a_0)r(s_m,a_0)\right)\\
    &= \sum_{a_0\in\A}q_{d_0(s_0)}(a_0)\sum_{s_m\in\St} (P_{\mu_0} \cdots P_{\mu_{m-1}})(s_0, s_m)r(s_m,a_0)\\
    &= \sum_{a_0\in\A}q_{d_0(s_0)}(a_0)(P_{\mu_0} \cdots P_{\mu_{m-1}}r_{a_0})(s_0),
\end{align*}
where for all $a\in\A$, $r_a:= (r(s,a))_{s\in\St}\in\mathbb{R}^{\St}$ is the reward vector corresponding to a given action. 
Next,  we rewrite the second term in \eqref{eq: rewriting the value} as
\begin{align*}
     &\sum_{ t= m+1}^{\infty} \gamma^{t-m} \sum_{\substack{ s_1,  \cdots, s_t \in\St} }R_{d_{t-m}}(s_t, s_{t-m})
\left(\prod_{k=0}^{m-1} P_{\mu_k}(s_k, s_{k+1})\right)\left(\prod_{k=m}^{t-1} P_{d_{k-m}(s_{k-m})}(s_k, s_{k+1})\right)\\
=& \sum_{ t= m+1}^{\infty} \gamma^{t-m}\sum_{\substack{ s_1,  \cdots, s_t \in\St} }R_{d_{t-m}}(s_t, s_{t-m}) \left(\prod_{k=0}^{m-1} P_{\mu_k}(s_k, s_{k+1})\right)\\
&\qquad P_{d_{0}(s_{0})}(s_m, s_{m+1})\left(\prod_{k=m+1}^{t-1} P_{d_{k-m}(s_{k-m})}(s_k, s_{k+1})\right)\\
=& \sum_{ t= m+1}^{\infty} \gamma^{t-m}\sum_{\substack{ s_1,  \cdots, s_t \in\St} }R_{d_{t-m}}(s_t, s_{t-m}) \left(\prod_{k=0}^{m-1} P_{\mu_k}(s_k, s_{k+1})\right)\\
&\qquad \left(\sum_{a_0\in\A}q_{d_0(s_0)}(a_0)P(s_{m+1}|s_m, a_0)\right)\left(\prod_{k=m+1}^{t-1} P_{d_{k-m}(s_{k-m})}(s_k, s_{k+1})\right)\\
=& \sum_{a_0\in\A}q_{d_0(s_0)}(a_0) \left[\sum_{ t= m+1}^{\infty} \gamma^{t-m}\sum_{\substack{ s_1,  \cdots, s_t \in\St} }R_{d_{t-m}}(s_t, s_{t-m}) \left(\prod_{k=0}^{m-1} P_{\mu_k}(s_k, s_{k+1})\right)\right.\\
&\qquad  \left.P(s_{m+1}|s_m, a_0)\left(\prod_{k=m+1}^{t-1} P_{d_{k-m}(s_{k-m})}(s_k, s_{k+1})\right)\right]
\end{align*}
Putting the two expressions together gives
\begin{align*}
   v^{\mu_0:\mu_{m-1},\pi}_{m}(s_0) 
&= \sum_{a_0\in\A}q_{d_0(s_0)}(a_0)\left[(P_{\mu_0} \cdots P_{\mu_{m-1}}r_{a_0})(s_0) + \sum_{ t= m+1}^{\infty} \gamma^{t-m}\sum_{\substack{ s_1,  \cdots, s_t \in\St} }R_{d_{t-m}}(s_t, s_{t-m})\right.\\ &\qquad \left.\left(\prod_{k=0}^{m-1} P_{\mu_k}(s_k, s_{k+1})\right)P(s_{m+1}|s_m, a_0)\left(\prod_{k=m+1}^{t-1} P_{d_{k-m}(s_{k-m})}(s_k, s_{k+1})\right) \right]\\
&\overset{(1)}{\leq} \max_{a_0\in\A}\left\{(P_{\mu_0} \cdots P_{\mu_{m-1}}r_{a_0})(s_0) + \sum_{ t= m+1}^{\infty} \gamma^{t-m}\sum_{\substack{ s_1,  \cdots, s_t \in\St} }R_{d_{t-m}}(s_t, s_{t-m})\right.\\ &\qquad \left.\left(\prod_{k=0}^{m-1} P_{\mu_k}(s_k, s_{k+1})\right)P(s_{m+1}|s_m, a_0)\left(\prod_{k=m+1}^{t-1} P_{d_{k-m}(s_{k-m})}(s_k, s_{k+1})\right) \right\}\\
&\overset{(2)}{:=} v^{\mu_0:\mu_{m-1},\pi^0_{\textsc{D}}}_{m}(s_0), 
\end{align*}
where $(1)$ holds by applying \citep{puterman2014markov}[Lemma 4.3.1], and $(2)$ by defining 
policy $\pi^0_{\textsc{D}}:= (d^{\textsc{D}}_0, d_1,d_2\cdots)\in\Pi^{\textsc{MR}}$ such that 
the first decision rule is deterministic  $d^{\textsc{D}}_0:=\delta_{a_0^*}$ with $a_0^*$ the argmax of $(1)$, while $(d_1,d_2,\cdots)$ are the same as in the original policy $\pi\in\Pi^{\textsc{MR}}$. 

We now continue to showing the dependence of the value on $d_i(s_i)$ for $i\geq1,$ by continuing with the second term in \eqref{eq: rewriting the value}:
\begin{align*}
    &\sum_{ t= m+1}^{\infty} \gamma^{t-m}\sum_{ s_1,  \cdots, s_t \in\St }R_{d_{t-m}}(s_t, s_{t-m})\left(\prod_{k=0}^{m-1} P_{\mu_k}(s_k, s_{k+1})\right)P(s_{m+1}|s_m, a_0^*)\left(\prod_{k=m+1}^{t-1} P_{d_{k-m}(s_{k-m})}(s_k, s_{k+1})\right)\\
    &= \gamma\sum_{ s_1,  \cdots, s_{m+1} \in\St }R_{d_{1}}(s_{m+1}, s_{1})\left(\prod_{k=0}^{m-1} P_{\mu_k}(s_k, s_{k+1})\right)P(s_{m+1}|s_m, a_0^*) \\
    &\qquad +\sum_{ t= m+2}^{\infty} \gamma^{t-m}\sum_{ s_1,  \cdots, s_t \in\St }R_{d_{t-m}}(s_t, s_{t-m})\left(\prod_{k=0}^{m-1} P_{\mu_k}(s_k, s_{k+1})\right)P(s_{m+1}|s_m, a_0^*)\\
    &\qquad\qquad P_{d_{1}(s_{1})}(s_{m+1}, s_{m+2}) \left(\prod_{k=m+2}^{t-1} P_{d_{k-m}(s_{k-m})}(s_k, s_{k+1})\right)\\
    &= \sum_{a_1\in\A}q_{d_1(s_1)}(a_1) \left[ \gamma\sum_{ s_1,  \cdots, s_{m+1} \in\St }r(s_{m+1}, a_{1})\left(\prod_{k=0}^{m-1} P_{\mu_k}(s_k, s_{k+1})\right)P(s_{m+1}|s_m, a_0^*) \right.\\
    &\qquad\qquad\qquad\qquad +\sum_{ t= m+2}^{\infty} \gamma^{t-m}\sum_{ s_1,  \cdots, s_t \in\St }R_{d_{t-m}}(s_t, s_{t-m})\left(\prod_{k=0}^{m-1} P_{\mu_k}(s_k, s_{k+1})\right)P(s_{m+1}|s_m, a_0^*)\\
    &\qquad\qquad\qquad\qquad~~ \left.P(s_{m+2}| s_{m+1}, a_1) \left(\prod_{k=m+2}^{t-1} P_{d_{k-m}(s_{k-m})}(s_k, s_{k+1})\right)\right]\\
    &\leq \max_{a_1\in\A}\left\{\gamma\sum_{ s_1,  \cdots, s_{m+1} \in\St }r(s_{m+1}, a_1)\left(\prod_{k=0}^{m-1} P_{\mu_k}(s_k, s_{k+1})\right)P(s_{m+1}|s_m, a_0^*) \right.\\
    &\qquad \qquad \left. +\sum_{ t= m+2}^{\infty} \gamma^{t-m}\sum_{ s_1,  \cdots, s_t \in\St }R_{d_{t-m}}(s_t, s_{t-m})\left(\prod_{k=0}^{m-1} P_{\mu_k}(s_k, s_{k+1})\right)P(s_{m+1}|s_m, a_0^*)\right.\\
    &\qquad\qquad\quad\left. P(s_{m+2}|s_{m+1}, a_1)\left(\prod_{k=m+2}^{t-1} P_{d_{k-m}(s_{k-m})}(s_k, s_{k+1})\right)\right\}\\
    &\vdots\\
    &\leq \sum_{t=m+1}^{\infty}  \gamma^{t-m}\sum_{ s_1,  \cdots, s_t \in\St }r(s_t,a_{t-m}^* )\left(\prod_{k=0}^{m-1} P_{\mu_k}(s_k, s_{k+1})\right)\left(\prod_{k=m+1}^{t-1} P( s_{k+1}|s_k, a_{k-m}^*)\right).
\end{align*}
Let the deterministic decision rule $d^{\textsc{D}}_i:= \delta_{a^*_i}$ with $a^*_i$ being the optimal action per each maximization above for every $i\geq1,$ and the resulting deterministic policy $\pi_{\textsc{D}}:= (d^\textsc{D}_0,d^\textsc{D}_1,\cdots).$ Then,
\begin{align*}
    v^{\mu_0:\mu_{m-1},\pi}_{m}(s_0) \leq v^{\mu_0:\mu_{m-1},\pi_{\textsc{D}}}_{m}(s_0),  
\end{align*}
i.e.,
\begin{align*}
    \max_{\pi\in\Pi^{\textsc{MR}}}v^{\mu_0:\mu_{m-1},\pi}_{m}(s_0) \leq \max_{\pi\in\Pi^{\textsc{MD}}}v^{\mu_0:\mu_{m-1},\pi}_{m}(s_0).  
\end{align*}

\end{proof}

\section{Experiments}
\label{appx: implementation details}
\subsection{Numerical Summary of Atari Results}
\label{sec: atari table}
\begin{table}[h]
{\tiny
\hspace{-2cm}
  \begin{tabular}{|l|l|c|c|c|c|c|c|c|c|c|}
    \hline &
    \multirow{3}{*}{Environment} &
      \multicolumn{3}{c}{$m=5$} &
      \multicolumn{3}{c}{$m=15$} &
      \multicolumn{3}{c|}{$m=25$}\\
    && Del. & Aug. & Obl. & Del. & Aug. & Obl. & Del. & Aug. & Obl. \\
    \hline
    \multirow{2}{*}{Tabular} &
    Maze & ${\bf 0.50} \pm 0.49$ & $0.18 \pm 0.54$ & $-0.61 \pm 0.26$ & ${\bf -0.25} \pm 0.43$ & $-1 \pm 0$ & $-0.76 \pm 0.16$ & ${\bf -0.45} \pm 0.36$ & N/A & $-0.71 \pm 0.19$ \\
    
    &Noisy Maze & ${\bf 0.40} \pm 0.44$ & $-0.29 \pm 0.45$ & $-0.64 \pm 0.21$ & ${\bf -0.50} \pm 0.26$ & $-1 \pm 0$ & $-0.78 \pm 0.15$ & ${\bf -0.49} \pm 0.34$ & N/A & $-0.99 \pm 0$ \\
    \hline
    \multirow{3}{*}{Physical} &
    Cartpole & ${\bf 489} \pm 11$ & $453 \pm 16$ & $27 \pm 4$ & ${\bf 414} \pm 14$ & $192 \pm 15$ & $30 \pm 3$ & ${\bf 324} \pm 7$ & $41 \pm 2$ & $41 \pm 3$ \\
    
    &Noisy Cartpole & ${\bf 435} \pm 8$ & $379 \pm 17$ & $26 \pm 3$ & ${\bf 251} \pm 22$ & $129 \pm 24$ & $30 \pm 3$ & ${\bf 60} \pm 7$ & $36 \pm 3$ & $40 \pm 3$ \\
    &Acrobot & ${\bf -131} \pm 32$ & $-463 \pm 18$ & $-467 \pm 47$ & ${\bf -211} \pm 53$ & $-481 \pm 21$ & $-467 \pm 34$ & ${\bf -351} \pm 57$ & $-493 \pm 5$ & $-465 \pm 20$ \\
    &Noisy Acrobot & ${\bf -134} \pm 37$ & $-491 \pm 2$ & $-445 \pm 11$ & ${\bf -329} \pm 24$ & $-425 \pm 41$ & $-399 \pm 41$ & ${\bf -361} \pm 62$ &  $-471 \pm 12$ & $-438 \pm 39$ \\
    \hline
    													
     \multirow{10}{*}{Atari} &
    Enduro & $16 \pm 6$ & $29 \pm 4$ & ${\bf 33} \pm 2$ & ${\bf 1.4} \pm 0.4$ & $0.6 \pm 0.7$ & $0.5 \pm 0.2$ & ${\bf 1.1} \pm 0.6$ & $ 0.2 \pm 0.1$ & $0.2 \pm 0.4$ \\
    &MsPacman & ${\bf 1354} \pm 86$ & $1083 \pm 60$ & $1319 \pm 35$ & $ {\bf1034}	\pm 124$ & $ 691 \pm 272$ & $ 701 \pm 123$ & ${\bf 959} \pm 77$ & $450 \pm 84$ & $612 \pm 23$ \\
    &NameThisGame & ${\bf2476} \pm 96$ & $2278 \pm 167$ & $2153 \pm 152$ & ${\bf 2122} \pm  132$ & $1573 \pm 43$ & $2013 \pm 300$ & ${\bf 1887} \pm 204$ & $1510 \pm 210$ & $1775 \pm 96$ \\
    &Qbert & $367 \pm 19$ & $372 \pm 177	$ & ${\bf 402} \pm 152$ & ${\bf 304} \pm 15$ & $245 \pm 29$ & $254 \pm 34	$ & ${\bf 253} \pm 29$ & $154 \pm	77$ & $200 \pm 74$ \\
    &RoadRunner & ${\bf 2975} \pm 237$ & $1790 \pm 255$ & $1152 \pm 430$ & ${\bf 1294} \pm 472$ & $1153 \pm 119$ & $360	\pm	204$ & ${\bf 1056} \pm 698$ & $668 \pm 268$ & $485	\pm	451$ \\
    &StarGunner & $902 \pm 74$ & $838 \pm 104$ & ${\bf919} \pm	44$ & ${\bf 801} \pm 38$ & $622 \pm 68$ & $643 \pm 50$ & ${\bf 712} \pm 49$ & $649 \pm 47$ & $635 \pm 20$ \\
    &TimePilot & ${\bf 1941} \pm 133$ & $1844  \pm 599$ & $1616 \pm  474$ & ${\bf 2695} \pm 418$ & $2049 \pm  665$ & $2341 \pm  72$ & ${\bf 2690} \pm 201$ & $2671 \pm  127$ & $1980 \pm  623$ \\
    &Zaxxon & ${\bf 1418} \pm 148$ & $431 \pm 77$ & $605 \pm 66$  & ${\bf 461} \pm 185$ & $97 \pm 65$ & $225 \pm 19	$ & ${\bf 130} \pm 42$ & $	72 \pm 22$ & $	67 \pm 35$ \\
    \hline
  \end{tabular}
  }
  \caption{Experiment summary: episodic return mean and std for all domains. Delayed-Q outperforms the alternatives in 39 of 42 experiments.}
  \label{table: exp summary}
\end{table}
 \newpage
 \subsection{Comparison to RNN-Based Policy}
 \label{sec: rnn comparison}
In one environment, we compared Delayed-Q with a fourth algorithm which uses an RNN-based policy that is unaware of the delay value. Specifically, we tested A2C, which managed to converge on Atari's Frostbite. As can be seen in Fig.~\ref{fig: rnn comparison}, using a recurrent policy does not improve upon Augmented-Q or Oblivious-Q. This result is not surprising though: as stated in Thm.~\ref{theorem: markov policy is sufficient}, the sequence of states $s_{t-m}, s_{t-m-1}, \dots$ does not aid the policy any further than only using $s_{t-m}$. An additional deficiency of RNN-policies that are oblivious to the delay value is that, similarly to Oblivious-Q, they target the wrong Q-value without accounting for delayed execution. Notice that this is not the case in both Augmented-Q and Delayed-Q. 
\begin{figure}[!ht]
    \centering
    \includegraphics[scale=0.2]{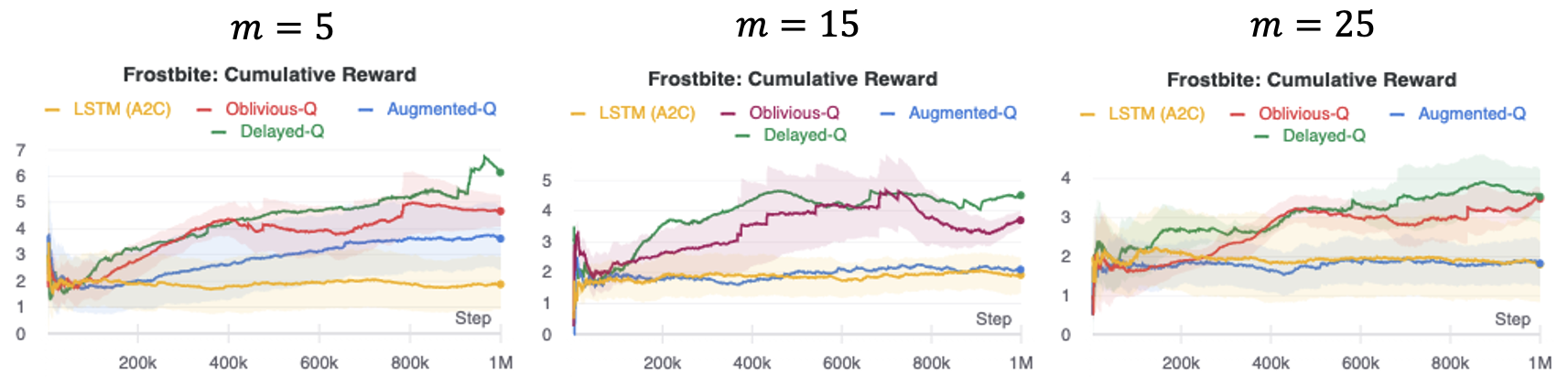}
    \caption{Comparison to RNN-based policy on Atari ``Frostbite''.}
    \label{fig: rnn comparison}
\end{figure}

\end{document}